\documentclass[10pt,twocolumn,letterpaper]{article}
\usepackage[T1]{fontenc}
\usepackage[utf8]{inputenc}
\usepackage[english]{babel}
\usepackage{lmodern}
\usepackage[margin=0.8in]{geometry}
\usepackage{microtype}
\usepackage{amsmath,amssymb,mathrsfs,mathtools,amsthm,bbm}
\usepackage{centernot}
\usepackage{graphicx}
\usepackage{algorithm,algorithmic}
\usepackage{xcolor}
\usepackage[authoryear,round]{natbib}

\usepackage{enumitem}
\usepackage{booktabs,array}
\usepackage{caption}
\usepackage{xurl}
\usepackage[hidelinks]{hyperref}
\usepackage[nameinlink,noabbrev]{cleveref}
\usepackage{authblk}

\hypersetup{
pdftitle={Safe Meta-Policy Design with Risk Control}, pdfauthor={Wenbin Zhou; Michael Lingzhi Li; Shixiang Zhu}, pdfsubject={Safe meta-policy design and deployment scheduling}, pdfkeywords={policy learning, deployment scheduling, risk control}
}

\title{Safe Meta-Policy Design with Risk Control}
\author[1]{Wenbin Zhou}
\author[2]{Michael Lingzhi Li}
\author[1]{Shixiang Zhu}
\affil[1]{Heinz College of Information Systems and Public Policy, Carnegie Mellon University}
\affil[2]{Technology and Operations Management, Harvard Business School}
\date{}

\renewcommand{\d}{\mathrm{d}}

\newtheorem{theorem}{Theorem}
\newtheorem{lemma}{Lemma}
\newtheorem{proposition}{Proposition}
\newtheorem{corollary}{Corollary}
\theoremstyle{definition}

\newtheorem{assumption}{Assumption}

\theoremstyle{remark}

\crefname{theorem}{theorem}{theorems}
\Crefname{theorem}{Theorem}{Theorems}
\crefname{lemma}{lemma}{lemmas}
\Crefname{lemma}{Lemma}{Lemmas}
\crefname{proposition}{proposition}{propositions}
\Crefname{proposition}{Proposition}{Propositions}
\crefname{corollary}{corollary}{corollaries}
\Crefname{corollary}{Corollary}{Corollaries}

\newcommand*{\ie}{\emph{i.e.}{}}

\begin{document}

\twocolumn[
\begin{@twocolumnfalse}
\maketitle
\vspace{-4ex}
\begin{abstract}
Models can be retrained as new data arrive, but deploying every new version risks replacing a good policy with a worse one. We study how to plan policy updates (i.e., meta-policy) before future candidates are trained, balancing the benefits of improvement against the risk of performance regression. Our offline meta-policy maximizes expected cumulative value subject to a budget on the expected number of updates that perform worse than the policies they replace. We estimate the value and risk of possible switches from historical learning trajectories, represent an update schedule as a path in a directed acyclic graph, and select a schedule using dynamic programming. A leading-order analysis identifies the signal-to-noise ratio of policy improvement as a key driver of update frequency, waiting times, and risk allocation: clearer improvements support earlier, more frequent updates, while noisier improvements call for longer waits or greater risk expenditure. Their asymptotic rates also reveal a diminishing marginal cost of achieving greater safety over time. Experiments on synthetic and clinical trial data illustrate the performance--risk tradeoff and compare our method with alternative baselines.
\end{abstract}
\end{@twocolumnfalse}
\vspace{1em}
]

\section{Introduction}

\begin{figure*}[t]
    \centering
    \includegraphics[width=1.0\linewidth]{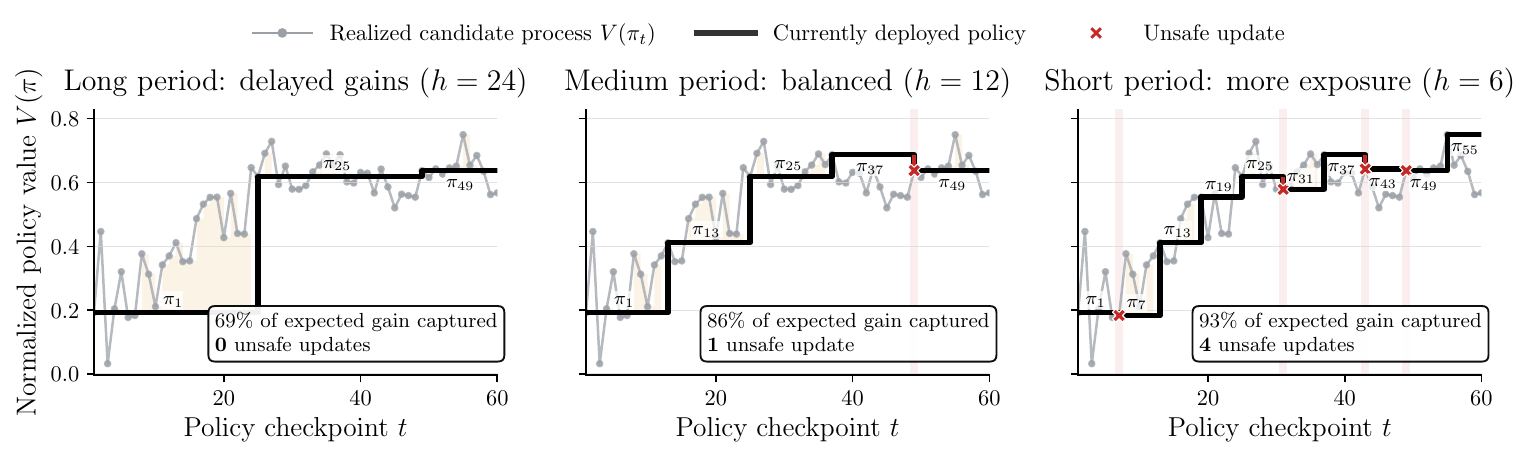}
\caption{The core tradeoff in deciding how often to update a policy. Updating slowly (left) leaves an outdated policy in place and loses value, while updating quickly (right) tracks improvements closely but risks deploying a policy worse than the one it replaces (a safety violation).}
    \label{fig:intro}
\end{figure*}

Machine learning systems are often retrained as new data become available, but a newly trained model need not be better than the one currently deployed. Organizations must therefore decide both how to learn from additional data and when to put the resulting models into use. These decisions arise in applications such as treatment recommendations, clinical decision support, lending, and pricing, where model updates change the decisions made for users \citep{qian2011performance,zhao2012estimating,athey2021policy}.

Choosing when to update involves a clear tradeoff. Although additional data may improve performance in expectation, sampling variability can produce a candidate whose performance is lower than that of the deployed model \citep{pirotta2013safe}. As illustrated in \Cref{fig:intro}, frequent updates deliver improvements sooner but create more opportunities for performance regression. Waiting reduces this exposure at the cost of delaying useful improvements. The challenge is to choose a deployment schedule that balances these gains and risks \citep{zhou2026conformalized,zhou2026calibrating}.

In many applications, this deployment process must be specified before the performance of future candidate models is known. For example, the U.S. Food and Drug Administration's guidance on Predetermined Change Control Plans allows manufacturers of AI-enabled medical devices to specify how anticipated modifications will be developed and evaluated before implementation \citep{fda2025pccp_ai_dsf}. Motivated by such settings, we study an offline planning problem in which models are trained sequentially and their deployment times are chosen in advance. We call this schedule a \emph{meta-policy} and define an update as \emph{unsafe} if it lowers average policy value in the target population. Our goal is to maximize expected cumulative deployed value subject to a budget on the expected number of unsafe updates.

Each update changes the deployed policy and therefore the risk of later updates. Deploying a candidate makes it the benchmark for the next replacement, whereas skipping it leaves the current policy in place. We therefore optimize the deployment schedule as a whole. When candidate values improve in expectation, we show that an optimal schedule deploys the newest available candidate whenever it updates. This result allows us to represent a schedule as a path through candidate training times. Each edge records the expected value earned while a policy remains deployed and the risk of replacing it at the next update. We estimate these quantities from historical learning trajectories and use dynamic programming to select the path with the highest expected value subject to the safety budget \citep{beasley1989resource}.

We also study how the learning dynamics determine the optimal schedule. Under a time-inhomogeneous It\^o diffusion model for policy value \citep{ito1951stochastic}, we derive a leading-order continuous approximation in which update intervals are inversely proportional to the squared local signal-to-noise ratio (SNR) of policy improvement: More reliable improvements support more frequent updates, while noisier improvements require longer waits. For a horizon of $T$ periods and safety budget $\epsilon$, the analysis yields regret of order $\log(T/\epsilon)/T$ and an update count of order $T/\log(T/\epsilon)$, with risk allocated adaptively according to the local learning dynamics. These rates reveal a favorable scaling law: as the horizon grows, the meta-policy can deploy an increasing number of policy updates while the average value lost from delayed deployment still vanishes. At the same time, stricter safety requirements enter only through $\log(1/\epsilon)$, so multiplicative reductions in the allowable risk lead to only additive increases in waiting time and a correspondingly mild loss in deployed value. In sum, these results show that substantial gains in safety can be achieved without sacrificing responsiveness at the same rate.

Experiments in synthetic learning environments and on International Stroke Trial data \citep{ist1997international,sandercock2011international} illustrate these tradeoffs. The proposed schedules achieve higher value under the estimated safety budget than periodic updating and alternatives that restrict update timing or risk allocation. The gains are most pronounced when the rate and variability of policy improvement change over time, making a fixed update interval less effective.

\paragraph{Related Work}

Safe policy improvement seeks to replace a baseline policy only when the available evidence supports non-inferiority or improvement. Early approaches certify a learned candidate using held-out data or construct conservative updates from policy-improvement bounds \citep{thomas2015high,laroche2019safe}, while conservative bandit methods maintain performance relative to a default policy during learning \citep{wu2016conservative}. Recent work has strengthened this perspective by optimizing lower-tail performance guarantees \citep{zhao2024absolute}, calibrating multiple tests over candidate threshold policies \citep{cho2025cspimt}, restricting deviations to well-supported decision points \citep{sharma2025decision}, and conformally regulating deviations from a safe reference policy \citep{prinster2026conformal}. While these methods primarily construct or certify an individual policy relative to a baseline, our framework takes a sequence of data-dependent candidate policies as given and jointly selects their deployment times, thereby extending safe policy improvement from a single-update decision to a sequential deployment problem.

A complementary literature studies sequential learning when changing the deployed policy is costly or when only a limited number of data-collection batches is available. Low-switching reinforcement-learning algorithms control regret while limiting policy changes \citep{bai2019switching}, and deployment-efficient methods minimize the number of distinct policies required to learn a high-performing policy \citep{matsushima2021deployment,huang2022deployment}. Recent work extends these guarantees to broad sequential-decision classes and batch-learning constraints \citep{xiong2024adaptivity}, generalized linear contextual bandits with update budgets \citep{sawarni2024generalized}, reinforcement learning with general function approximation \citep{zhao2024lowswitching}, and reward-free exploration with near-optimal deployment complexity \citep{zhang2025deployment}. Whereas this literature generally minimizes regret or sample complexity under an exogenous switching or deployment constraint, our meta-policy makes the number and timing of updates endogenous and explicitly weighs cumulative value against the probability that each deployed candidate performs worse than its incumbent, giving rise to a distinct set of structural properties.

Our computational formulation is related to the resource-constrained shortest-path problem, for which classical work developed exact dynamic programs and approximation schemes \citep{handler1980dual,beasley1989resource,hassin1992approximation,lorenz2001simple}. Recent advances include multiphase dynamic programming for acyclic networks \citep{himmich2024multiphase}, enhanced bidirectional heuristic search \citep{ahmadi2025bidirectional}, and faster label-setting methods that accommodate multiple or signed resource quantities \citep{ahmadi2025fast}. Our theoretical analysis is also connected to continuous-time approximations of stochastic-gradient algorithms, including stochastic modified equations and local diffusion models \citep{li2017stochastic,mandt2017sgd}; recent refinements study stochastic modified flows and the weak accuracy of alternative stochastic differential equation approximations \citep{gess2024flows,ankirchner2024comparison}. Unlike generic path solvers or diffusion analyses of parameter iterates, our contribution derives a problem-specific graph whose edges encode estimated deployment value and unsafe-update risk, and then transports the learning dynamics to the policy-value process to characterize the optimal number, timing, and risk allocation of policy updates.

\section{Problem Setup}
\label{subsec:DMenvironment}

We consider sequential policy learning over a finite horizon of $T \ge 2$ periods, indexed by $\mathcal{T}=\{1,\dots,T\}$. In period $t$, the decision-maker observes $(X_t,A_t,Y_t)$, where $X_t\in\mathcal{X}$ denotes pre-treatment covariates, $A_t\in\{0,1\}$ is the assigned treatment, and $Y_t\in\mathbb{R}$ is the observed outcome. Under the potential outcomes framework \citep{rubin1974estimating}, let $Y_t(0)$ and $Y_t(1)$ denote the outcomes under the two treatments, with consistency requiring $Y_t=Y_t(A_t)$. We impose the following assumption on the underlying population \citep{jia2024cramming}.
\begin{samepage}
\begin{assumption}
\label{asm:iid}
The covariates and potential outcomes are drawn i.i.d. from $\mathcal{P}$, \ie,
\[
\left\{(X_t,Y_t(0),Y_t(1))\right\}_{t=1}^{T}
\stackrel{\mathrm{i.i.d.}}{\sim}\mathcal{P}.
\]
\end{assumption}
\end{samepage}
This assumption concerns the covariates and potential outcomes; treatment assignments need not be independent across periods.

A policy is a mapping $\pi:\mathcal{X}\to\{0,1\}$, with value
\[
V(\pi)=\mathbb{E}_{\mathcal{P}}\!\left[Y(1)\mathbbm{1}\{\pi(X)=1\}+Y(0)\mathbbm{1}\{\pi(X)=0\}\right].
\]
At each period $t$, the decision-maker trains a candidate policy $\pi_t$ using the observations accumulated through that period, $ \mathcal{D}_t=\{(X_\tau,A_\tau,Y_\tau)\}_{\tau=1}^{t}. $ For a learned policy, $V(\pi_t)$ evaluates its population performance conditional on the realized training data.

Candidate training and deployment are governed separately. A \emph{meta-policy} $m$ specifies which candidate is deployed in each period: the deployed policy at time $t$ is $\pi_{m(t)}$. We consider schedules fixed in advance, with feasible class
\[
\mathcal{M}=\left\{m:\mathcal{T}\to\mathcal{T}:m(t)\le t\ \text{for all }t\in\mathcal{T}\right\}.
\]
The restriction $m(t)\le t$ ensures that a candidate is trained before it is deployed. We assume an \emph{offline} setting in which deployment decisions do not affect the data-generating process used to train future candidates. \Cref{fig:reform} illustrates an example of meta-policy.

An update is \emph{unsafe} if the newly deployed policy has lower population value than its incumbent. Given a safety budget $\epsilon>0$, we choose a meta-policy to maximize expected average deployed value while bounding the expected number of unsafe updates:
\begin{equation}
\label{eq:obj}
\begin{aligned}
\max_{m\in\mathcal{M}}\quad
&\mathbb{E}\!\left[\frac{1}{T}\sum_{t=1}^{T}V(\pi_{m(t)})\right]\\
\text{s.t.}\quad
&\mathbb{E}\!\left[\sum_{t=2}^{T}\mathbbm{1}\!\left\{V(\pi_{m(t)})<V(\pi_{m(t-1)})\right\}\right]\le\epsilon.
\end{aligned}
\end{equation}
Here, the outer expectation is over the training data $\mathcal{D}_T$, which generates the candidate-policy sequence. Appendix~\ref{app:generalize} discusses extensions beyond the binary-treatment causal setting.

\section{Proposed Algorithm}
\label{sec:reform}

We reformulate problem~\eqref{eq:obj} as a resource-constrained path problem and develop a dynamic programming algorithm using data-driven estimates of policy values and update risks.

\begin{figure}[t]
    \centering
    \includegraphics[width=\linewidth]{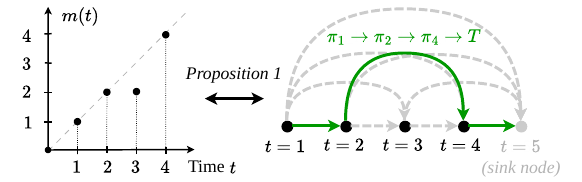}
\caption{A deployment schedule and its path representation for $T=4$. \textit{Left}: The meta-policy updates at times $2$ and $4$, retaining $\pi_2$ at time $3$. \textit{Right}: The corresponding path $1\to2\to4\to5$ deploys $\pi_1,\pi_2,\pi_4$, with node $5=T+1$ serving as the terminal sink.
    }
    \label{fig:reform}
\end{figure}

\subsection{Path Selection Reformulation}

The reformulation relies on the following assumption that candidate-policy values improve in expectation, motivated by policy-improvement methods \citep{pirotta2013safe,schulman2015trust}.
\begin{samepage}
\begin{assumption}[Policy improvement]
\label{asm:improve}
For all $t',t\in\mathcal{T}$ with $t'<t$, $ \mathbb{E}[V(\pi_{t'})]\le\mathbb{E}[V(\pi_t)]. $
\end{assumption}
\end{samepage}
\Cref{asm:improve} allows individual candidates to perform worse than their predecessors; it requires only that their expected values be nondecreasing. It therefore permits unsafe updates, unlike pathwise monotonicity.

The following result shows that an optimal schedule can be chosen to deploy the newest candidate whenever an update occurs.
\begin{samepage}
\begin{lemma}[Newest-policy reduction]
\label{lem:ss}
Under \Cref{asm:improve}, problem~\eqref{eq:obj} admits an optimizer $m^\star$ satisfying
\[
m^\star(t)\in\{m^\star(t-1),t\},
\qquad t\in\mathcal{T}\setminus\{1\}.
\]
\end{lemma}
\end{samepage}
By \Cref{lem:ss}, an optimal meta-policy can be specified entirely by its update times: once these times are chosen, the deployed candidate at each update is determined. Write $1=t_0<t_1<\cdots<t_K\le T$ for the initial deployment and subsequent updates, and set $t_{K+1}=T+1$ to mark the end of the horizon. The schedule deploys $\pi_{t_k}$ during periods $t_k,\dots,t_{k+1}-1$, for $k=0,\dots,K$. Thus, each consecutive pair $(t_k,t_{k+1})$ specifies a holding interval and, except for the final pair, the update that follows it.

This interval representation provides a natural connection to path optimization. Represent each candidate's training time by a vertex and each possible holding interval by a directed edge. A deployment schedule then corresponds to the path $t_0\to t_1\to\cdots\to t_K\to T+1$, as illustrated by \Cref{fig:reform}. Assigning each edge its contribution to expected deployed value and its unsafe-update probability makes both the objective and the safety constraint additive along the path. Formally, construct a directed acyclic graph $\mathcal{G}=(\mathcal{V},\mathcal{E})$ with
\[
\mathcal{V}=\{1,\dots,T+1\},
\quad
\mathcal{E}=\{(i,j):1\le i<j\le T+1\}.
\]
For $j\le T$, an edge $(i,j)$ represents deploying $\pi_i$ during periods $i,\dots,j-1$ and switching to $\pi_j$ at time $j$. The terminal edge $(i,T+1)$ represents retaining $\pi_i$ through period $T$, with no further update. Let $\mathcal{P}_{ij}$ denote the directed paths from $i$ to $j$, including the empty path when $i=j$:
\begin{equation}
\label{eq:path}
\begin{aligned}
    \mathcal P_{ij}
    & :=
    \big\{
    (v_0,\ldots,v_K):
    K\ge1,\; \\
     & \hspace{5ex}
    v_0=i,\;
    v_K=j,\;
    (v_{\ell-1},v_\ell)\in\mathcal E, \forall \ell
    \big\}.
\end{aligned}
\end{equation}

Each edge receives a reward equal to its contribution to expected average deployed value and a safety cost equal to its unsafe-update probability.

\begin{samepage}
\begin{proposition}[Path selection problem]
    \label{thm:knapsack}
Define $\mathcal{P}_{ij}$ as in \eqref{eq:path}. Under \Cref{asm:improve}, problem~\eqref{eq:obj} is equivalent to
    \begin{equation}
        \label{eq:rcspp}
            \max_{p \in \mathcal{P}_{1, T+1}}\quad
        \sum_{(i, j) \in p} \boldsymbol R_{ij} \quad \mathrm{s.t.}~
        \sum_{(i, j) \in p} \boldsymbol{S}_{ij} \leq \epsilon.
    \end{equation}
where  $\boldsymbol R \in \mathbb{R}^{(T + 1) \times (T + 1)}$ and $\boldsymbol S \in \mathbb{R}^{(T + 1) \times (T+1)}$ denote the reward and safety cost matrices,
    \[
    \begin{aligned}
    \boldsymbol R_{ij}&\coloneqq
            \frac{j - i}{T} \cdot \mathbb{E}[V(\pi_i)],\\
    \boldsymbol S_{ij}&\coloneqq
    \begin{cases}
    \mathbb{P}\left( V(\pi_j) < V(\pi_i) \right), & \text{if } j \leq T
     \\
    0, & \text{if } j = T + 1.
    \end{cases}
    \end{aligned}
    \]
\end{proposition}
\end{samepage}

Thus, an optimal schedule corresponds to a maximum-reward source--sink path subject to an additive safety budget. This is the maximization form of the resource-constrained shortest-path problem \citep{handler1980dual,beasley1989resource}. The safety costs add by linearity of expectation, without requiring independence between unsafe-update events.

\subsection{Dynamic Programming Solution}

We solve the path selection problem using plug-in edge estimates and a discretized safety budget. The resulting dynamic program is exact for the estimated, gridded problem. \Cref{alg:main} provides the pseudocode.

\paragraph{Step 1: Estimating Reward and Cost}

We estimate the edge rewards and safety costs using $N$ offline trajectories,
\[
\mathcal{D}^{(n)}=\{(X_t^{(n)},A_t^{(n)},Y_t^{(n)})\}_{t=1}^{T},
\qquad n=1,\dots,N.
\]
For each trajectory $n$, we train $\pi_i^{(n)}$ on the first $i$ observations using the same learning procedure that generates the candidate policies, and denote its estimated population value by $\widehat V_i^{(n)}$. The choice of policy-learning and value-estimation methods depends on how the data are collected \citep{dudik2014doubly,zhan2024policy,jia2024cramming}. Averaging across trajectories gives
\begin{equation}
\label{eq:estimation}
\begin{aligned}
\widehat{\boldsymbol{R}}_{ij}
&=
\frac{j-i}{NT}\sum_{n=1}^{N}\widehat V_i^{(n)},\\
\widehat{\boldsymbol{S}}_{ij}
&=
\begin{cases}
\displaystyle\frac{1}{N}\sum_{n=1}^{N}
\mathbbm{1}\!\left\{\widehat V_j^{(n)}<\widehat V_i^{(n)}\right\}, & j\le T,\\
0, & j=T+1.
\end{cases}
\end{aligned}
\end{equation}
Estimation error can cause a schedule that satisfies the empirical safety budget to violate the population constraint. Population risk control therefore requires accounting for uncertainty in the estimated safety costs.

\paragraph{Step 2: Dynamic Programming} To solve the estimated path problem, we first discretize the safety budget into units of size $\Delta>0$. We round the available budget down and each edge cost up:
\begin{equation}
\label{eq:discretized-constant}
B_\Delta=\left\lfloor\frac{\epsilon}{\Delta}\right\rfloor,
\qquad
\widetilde{\boldsymbol{S}}_{ij}
=
\left\lceil\frac{\widehat{\boldsymbol{S}}_{ij}}{\Delta}\right\rceil.
\end{equation}
This guarantees that any path using at most $B_\Delta$ budget units satisfies the estimated safety constraint:
\[
\sum_{(i,j)\in p}\widehat{\boldsymbol{S}}_{ij}
\le
\Delta\sum_{(i,j)\in p}\widetilde{\boldsymbol{S}}_{ij}
\le
\Delta B_\Delta
\le\epsilon.
\]
A smaller $\Delta$ reduces the error introduced by rounding but increases the computational cost.

For each node $j$ and budget $b\in\{0,\dots,B_\Delta\}$, let $F(j,b)$ denote the largest estimated reward achievable by a path from node $1$ to node $j$ using at most $b$ budget units. Initialize $F(1,b)=0$ for all $b$, and set the value of any infeasible state to $-\infty$. A path reaching $j$ through edge $(i,j)$ earns reward $\widehat{\boldsymbol{R}}_{ij}$ and leaves at most $b-\widetilde{\boldsymbol{S}}_{ij}$ budget units for the preceding path. For $j=2,\dots,T+1$, maximizing over predecessor nodes gives the Bellman recursion
\begin{equation}
\label{eq:bellman-recursion}
F(j,b)
=
\max_{\substack{(i,j)\in\mathcal{E}\\
\widetilde{\boldsymbol{S}}_{ij}\le b}}
\left\{
F\!\left(i,b-\widetilde{\boldsymbol{S}}_{ij}\right)
+\widehat{\boldsymbol{R}}_{ij}
\right\}.
\end{equation}
Because every edge connects an earlier node to a later one, we evaluate these states in increasing order of $j$. The optimal value of the discretized problem is $F(T+1,B_\Delta)$. Recording the selected predecessor at each state allows us to recover the optimal path and its deployment schedule by backtracking from $(T+1,B_\Delta)$.

Given the edge estimates, the algorithm requires $\mathcal{O}(T^2(B_\Delta+1))$ time and $\mathcal{O}(T^2+T(B_\Delta+1))$ memory, including storage for edge estimates, state values, and predecessors. Appendix~\ref{app:dp} provides further discussion of grid-width selection, approximation error, and an extension that accounts for uncertainty in the estimated safety costs.

\begin{algorithm}[t!]
\caption{Safe Meta-Policy Optimization}
\label{alg:main}
\begin{algorithmic}[1]
\REQUIRE Offline trajectories $\mathcal D$; time horizon $T$;
safety budget $\epsilon$; grid width $\Delta>0$.
\STATE Construct the update DAG $\mathcal G=(\mathcal V,\mathcal E)$
with source $1$ and destination $T+1$ as in \eqref{eq:rcspp}.
\STATE Estimate $\widehat{\boldsymbol R}$ and
$\widehat{\boldsymbol S}$ according to \eqref{eq:estimation}.
For each terminal edge $(i,T+1)$, set
$\widehat{\boldsymbol S}_{i,T+1}\leftarrow 0$.
\STATE Compute $B_\Delta$ and  $\widetilde{\boldsymbol S}_{ij}$
according to \eqref{eq:discretized-constant}.
\STATE Initialize
$F(1,b)\leftarrow 0$ for all $b\in\{0,\ldots,B_\Delta\}$ and
$F(j,b)\leftarrow-\infty$ for all $j\neq 1$ and
$b\in\{0,\ldots,B_\Delta\}$.
\FOR{each vertex $j\neq 1$ in topological order}
    \FOR{$b=0,\ldots,B_\Delta$}
        \STATE
        $F(j,b)\leftarrow
        \displaystyle
        \max_{\substack{(i,j)\in\mathcal E, ~
        \widetilde{\boldsymbol S}_{ij}\leq b}}
        \left\{
        F\!\left(i,b-\widetilde{\boldsymbol S}_{ij}\right)
        +\widehat{\boldsymbol R}_{ij}
        \right\}$.
        \STATE Store max predecessor when $F(j,b)>-\infty$.
    \ENDFOR
\ENDFOR
\STATE Recover optimal path $\hat p^\star$ by backtracking
from $(T+1, B_\Delta)$ through the stored predecessors.
\RETURN The meta-policy corresponding to  $\hat p^\star$.
\end{algorithmic}
\end{algorithm}

\section{Theoretical Analysis}
\label{sec:method}

This section characterizes and interprets the optimal update frequency, safety-risk allocation, and resulting deployment regret. All technical proofs are deferred to Appendix~\ref{app:kkt}.

Throughout our analysis, we assume that the policy value $V(\pi_t)$ as a time-inhomogeneous It\^o diffusion \citep{ito1951stochastic}, with deterministic drift $\mu(\cdot)$ and volatility $\sigma(\cdot)$.

\begin{samepage}
\begin{assumption}[It\^o diffusion]
    \label{ass:ito}
Let $\mu(\cdot): [0, 1] \to \mathbb R_+$ and $\sigma(\cdot): [0, 1] \to \mathbb R_+$ be two deterministic smooth functions, and $V_0$ be a constant. We assume that
    \begin{equation*}
        \begin{aligned}
        V(\pi_t)&=V_0+\frac{1}{T}\int_0^t \mu\left(\frac \tau T\right)\d\tau + \frac{1}{T}\int_0^t \sigma\left(\frac{\tau}{T}\right)\d W_\tau,
        \end{aligned}
    \end{equation*}
where $\{W_\tau\}_{\tau\geq0}$ is a standard Brownian motion.
\end{assumption}
\end{samepage}

This assumption is inspired by diffusion-type approximation, which has been used in the literature to study stochastic-gradient learning dynamics \citep{li2019stochasticmodified} and, more specifically, policy-gradient learning in stochastic bandits \citep{lattimore2026diffusion}. It can be derived from more primitive assumptions with $\pi_t$ itself, which we further discuss in Appendix~\ref{app:ito}.

\Cref{ass:ito} allows us to derive a continuous-time leading-order approximation to the original problem asymptotically. We introduce a horizon-dependent waiting-time function $h_T:[0,1]\to\mathbb{R}_+$ as the continuum analogue of the discrete holding intervals $t_{k+1}-t_k$. Starting from $t_0=0$, the associated continuous schedule is generated recursively by
\begin{equation}
    \label{eq:recursion}
    t_{k+1}=t_k+h_T(u_k),
    \qquad
    u_k\coloneqq t_k/T,
\end{equation}
for as long as the proposed next update satisfies $t_{k+1}<T$. Once the next proposed update would reach or exceed $T$, no further update is made, and the currently deployed policy is retained until the terminal time $T$. The update interval $h_T(u)$ is used in the approximated optimization problem as its decision variable.

\begin{samepage}
\begin{lemma}[Leading-order approximation]
    \label{lem:continuous}
Under \Cref{ass:ito}, as $T \to \infty$, Problem~\eqref{eq:obj} admits the following continuous leading-order approximation:
    \begin{equation}
        \label{eq:continuous}
        \begin{aligned}
        \min_{h_T(\cdot)}\quad&\int_0^1 \mu(u) h_T(u) \d u\\
        \mathrm{s.t.}
        \quad &
        \int_0^1
        \frac{
        \Phi\!\left(
        -\mu(u)
        \sqrt{h_T(u)} / \sigma(u)
        \right)
        }{
        h_T(u)
        }
        \d u
        \le
        \frac{\epsilon}{T}.
        \end{aligned}
    \end{equation}
\end{lemma}
\end{samepage}
The objective of \eqref{eq:continuous} is derived by replacing the original objective of maximizing cumulative reward with minimizing regret, where regret is measured relative to the always-update baseline,
\begin{equation}
    \label{eq:regret}
    \mathcal{R}_T^\star(\epsilon) \coloneqq \mathbb{E}\left[ \frac{1}{T}\sum_{t=1}^T \left\{ V(\pi_t) - V(\pi_{m^\star(t)}) \right\} \right].
\end{equation}
The integrand in the constraint can be interpreted as a local risk-expenditure rate, where the numerator is the probability that an update is harmful, and $1/h_T(u)$ captures the local update frequency. The constraint encourages frequent updates when learning is fast and stable, and longer waits when learning is slow or noisy.

\begin{samepage}
\begin{theorem}[Update interval]
    \label{thm:update-interval}
Define the following constants and functions,
    $$
    L_T = \log\left( \frac{T}{\epsilon} \right),
    \quad
    I(u) = \frac{\mu(u)^2}{2 \sigma(u)^2}.
    $$
Under the same assumptions as \Cref{lem:continuous}, there is
    $$
    h_T^\star(u) = \frac{1}{I(u)}\Big[ L_T-\frac32\log L_T+O(1) \Big].
    $$
\end{theorem}
\end{samepage}

\begin{figure*}[t]
    \centering
    \includegraphics[width=1.0\linewidth]{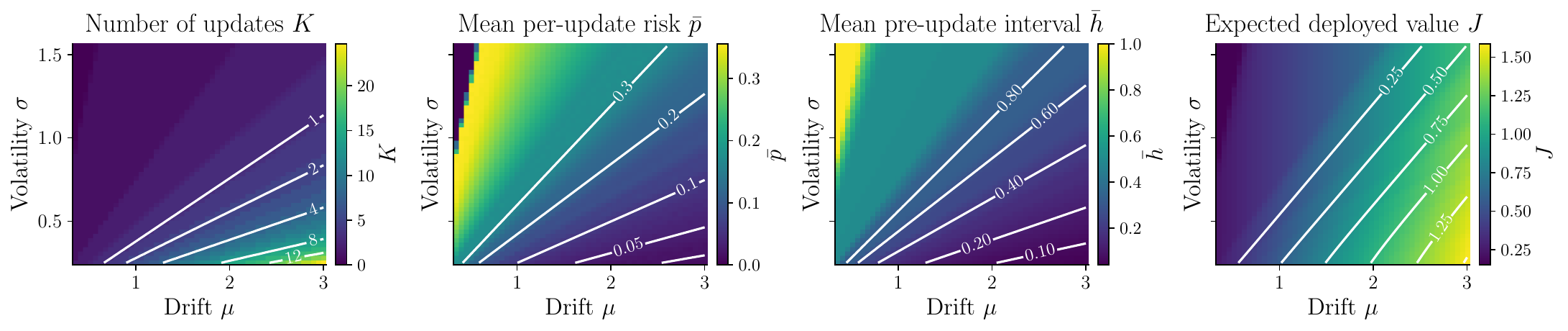}
\caption{ Phase diagrams of the schedules returned by \Cref{alg:main} on homogeneous value processes with constant drift $\mu$ and volatility $\sigma$ ($T=512$, $\epsilon=0.35$, $\Delta=1 \times 10^{-3}$). From left to right, the panels report the number of updates $K$, the mean per-update risk $\bar p$, the mean normalized pre-update interval preceding an update $\bar h$, and the expected deployed value $J$. Each cell is evaluated under the population rewards and safety costs and averaged over three simulated datasets. White-labeled contours show the leading-order theoretical predictions established in \Cref{sec:method}.
    }
    \label{fig:phase-diagram}
\end{figure*}

\begin{figure*}[t]
    \centering
    \includegraphics[width=1.0\linewidth]{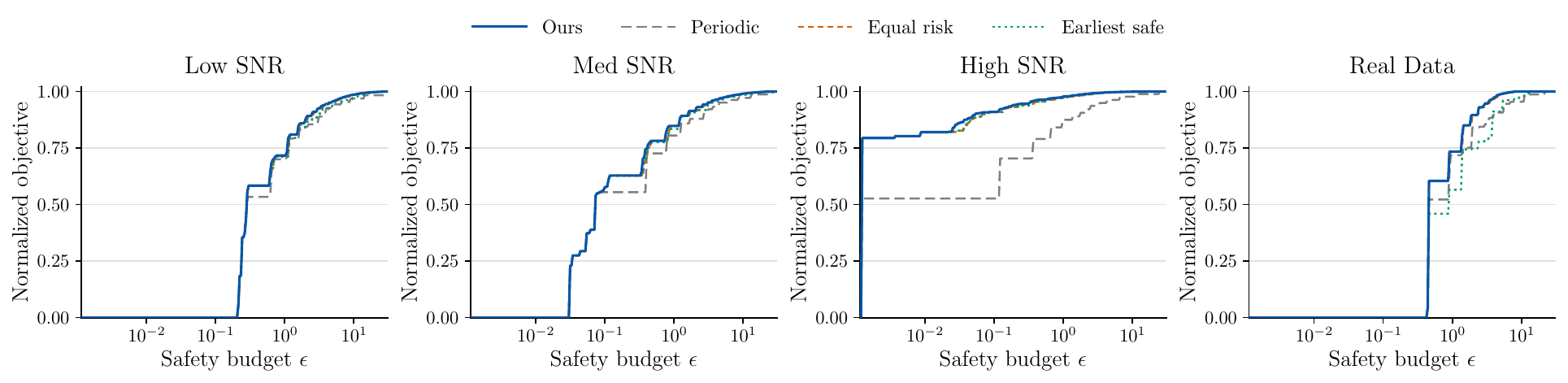}
\caption{ Safety--value Pareto frontiers of the proposed meta-policy (Ours) and three baseline update rules, in three synthetic settings with low, medium, and high value-process SNR and on the International Stroke Trial (IST) data. The horizontal axis is the safety budget $\epsilon$ (log scale). The vertical axis is the objective of \eqref{eq:obj}, normalized so that $0$ and $1$ correspond to never updating and to updating every period. All methods use the same estimated rewards and safety costs.
    }
    \label{fig:frontier}
\end{figure*}

\Cref{thm:update-interval} characterizes how the optimal update interval adapts to the local learning dynamics. Since $I(u)$ represents the squared local signal-to-noise ratio (SNR) and the leading term scales as $L_T/I(u)$, updates should occur more frequently when expected improvement is strong relative to its variability, and less frequently when learning is noisy. The logarithmic correction $-\tfrac{3}{2}\log L_T$ reflects the sharp-tail behavior governing unsafe-update risk.

Importantly, \Cref{thm:update-interval} reveals a favorable diminishing marginal cost of additional safety. Imposing stricter safety budgets increases the waiting time between updates only logarithmically through $L_T=\log(T/\epsilon)$. This implies that when the safety budget is relatively large, tightening the budget can initially require noticeable increases in the update interval; but as the safety requirement becomes increasingly stringent, further multiplicative reductions in $\epsilon$ translate into only additive increases in $L_T$. Consequently, a substantial improvement in deployment safety can be achieved without a proportionally large reduction in update frequency, encouraging practitioners in high-risk settings to adopt conservative update schedules without sacrificing responsiveness at the same rate.

\begin{samepage}
\begin{corollary}[Risk per-update]
    \label{cor:risk}
Under the same assumptions and notations as in \Cref{thm:update-interval}, let
    $$
    A(u) = \frac{\mu(u)}{I(u)},
    \quad
    B = \int_{0}^1A(u) \d u.
    $$
Then, there is
    $$
    \mathbb{P}\left\{ V(\pi_{t_{k+1}}) < V(\pi_{t_k}) \right\} = \Theta \left( \frac{\epsilon L_T}{TB}
    \frac{\mu(u_k)}{I(u_k)^2}
    \right),
    $$
where $u_k \coloneqq t_{k} / T$ denotes the normalized update times.
\end{corollary}
\end{samepage}
\Cref{cor:risk} shows that the optimal schedule allocates the safety budget unevenly across updates according to the local learning dynamics, with the per-update risk scales with $\mu(u)/I(u)^2$. Moreover, because the overall scale is $\epsilon L_T/T$, the risk associated with each individual update vanishes as the learning horizon grows, even while the number of updates increases. Hence, the optimal policy achieves global safety not by imposing a uniform risk level on every update, but by adaptively distributing risk across the trajectory.

\begin{samepage}
\begin{corollary}[Regret and number of updates]
\label{cor:regret}
\label{cor:K}
Let $R_T^\star(\epsilon)$ be defined in \eqref{eq:regret} and $K_T^\star$ be the total number of updates generated according to \eqref{eq:recursion}. Under the same assumptions and notations of \Cref{thm:update-interval}, there is
$$
\begin{aligned}
\mathcal R_T^\star(\epsilon)
& = \frac{B}{2T}\left[L_T-\frac32\log L_T\right]+O(T^{-1}), \\
K_T^\star & \sim\frac{T}{L_T} \int_0^1 I(u) \d u.
\end{aligned}
$$
\end{corollary}
\end{samepage}

\Cref{cor:regret} characterizes the aggregate cost and adaptivity of the optimal deployment schedule. In particular, the optimal regret vanishes at the rate $L_T/T$, while the number of updates grows on the order of $T/L_T$, showing that the system can accommodate an increasing number of policy updates over a longer learning horizon while incurring vanishing average deployment regret. The second-order terms further reveal that this tradeoff depends on the entire learning trajectory through $I(u)$: periods in which learning is more informative contribute differently to both cumulative regret and the total number of deployments. Together, these results quantify how the locally adaptive update rule translates into a global balance between responsiveness and the cost of holding outdated policies.

\section{Numerical Experiments}

\begin{figure*}[t]
    \centering
    \includegraphics[width=1.0\linewidth]{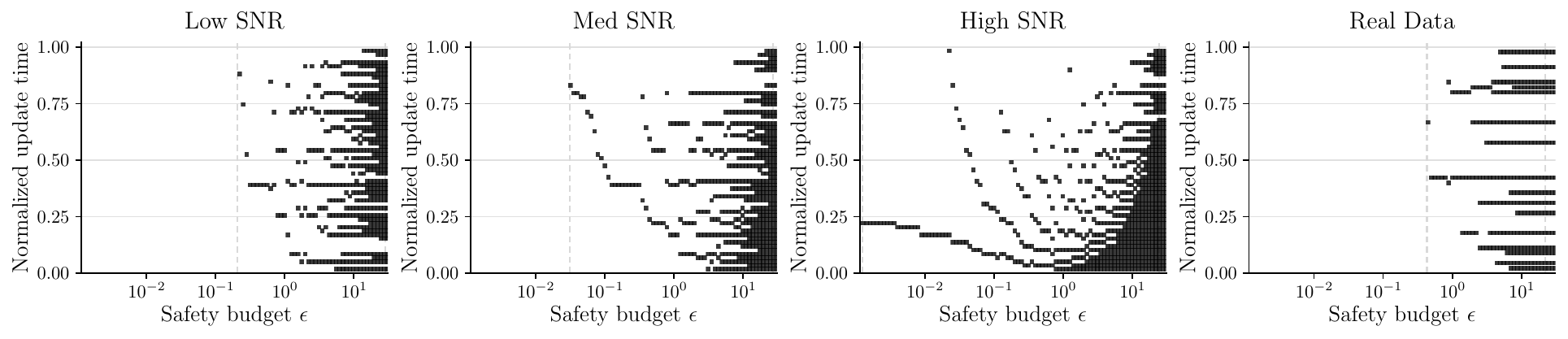}
\caption{ Update times in the optimal schedule selected by the proposed meta-policy as the safety budget $\epsilon$ varies, in the same four settings as \Cref{fig:frontier}. Each black dot indicates a scheduled normalized update time. The left dashed lines mark the smallest budget that admits one update.
    }
    \label{fig:bifurcation}
\end{figure*}

This section presents three sets of numerical experiments that compare the proposed meta-policy with update rules commonly adopted in practice and validate our theoretical analysis in \Cref{sec:method}.

\paragraph{Experiment setting.} We evaluate our framework on both synthetic and real-data policy-learning environments. For the synthetic study, we consider three learning regimes with low, medium, and high value-process signal-to-noise ratios by varying the learning rate and outcome-noise level. We generate 400 independent trajectories of 300 sequential observations and retain a candidate policy every five observations, yielding $T=60$ deployment checkpoints, five-dimensional covariates, randomized binary treatment, and heterogeneous treatment effects, and train an online plug-in CATE policy learner. For the real-data study, we use the randomized International Stroke Trial data \citep{sandercock2011international}, constructing $500$ bootstrap sequential learning trajectories with nested training sets and evaluating learned policies on a fixed $20\%$ held-out sample using an AIPW value estimator \citep{robins1994estimation}. Across these settings, we compare our proposed algorithm against ($i$) \emph{periodic} updating at the best feasible constant interval, ($ii$) \emph{equal risk}, which splits $\epsilon$ evenly across updates, and ($iii$) \emph{earliest safe}, which deploys the first policy whose estimated safety cost falls below a tuned threshold. Further experimental details are provided in Appendix~\ref{app:exp}.

\paragraph{Exp 1: Validating the optimality system.} \Cref{fig:phase-diagram}  uses simulated homogeneous value processes with constant drift $\mu$ and volatility $\sigma$ over a $50\times50$ grid of $(\mu,\sigma)$, runs \Cref{alg:main} on the estimated rewards and safety costs, and overlays their theoretical asymptotic rates as white contours. It can be observed that a higher SNR yields more updates, each spending less risk, and admits shorter waits. On the other hand, in the low SNR regime, there might not even exist an update that fits within the budget. The empirical and theoretical surfaces are closely aligned (Pearson correlations of $0.99$, $0.99$, and $0.98$ for $K$, $\bar p$, and $\bar h$). The experiment suggests that the theoretical structure characterized in \Cref{sec:method} aligns well with the exact schedules numerically computed by \Cref{alg:main}.

\paragraph{Exp 2: Safety--value frontier.} \Cref{fig:frontier} plots the obtained objective value of cumulative reward across different specifications of the safety budget $\epsilon$. It can be seen that our method weakly dominates every baseline at every budget, and its frontier exhibits the diminishing returns predicted by \Cref{cor:regret}. The gap is largest against periodic updating, reaching $0.39$ in the high-SNR setting, where our method captures $82\%$ of the always-update gain at $\epsilon=0.01$ versus $53\%$ for periodic updating; the adaptive baselines trail by up to $0.06$ (equal risk) and $0.23$ (earliest safe). These gains arise because our method jointly chooses when to update and how much risk to spend on each update, rather than fixing either in advance. This result shows that the proposed meta-policy attains a higher safety--value frontier than commonly used update rules in all four settings, rendering it superior in decision-making settings with safety constraints.

\paragraph{Exp 3: Structure of optimal schedules.} \Cref{fig:bifurcation} illustrates how the optimal schedule reorganizes as the safety budget grows, where each column shows the normalized update times selected under a given $\epsilon$. Three observations can be made: ($i$) \textit{The schedule is not nested in $\epsilon$}: as the budget grows, existing updates shift earlier, and new ones are inserted, producing the curved branches in the medium- and high-SNR panels. This agrees with \Cref{cor:risk}, where allocating more risk to an update shortens its waiting time. ($ii$) \textit{Update placement depends on the SNR}: under high SNR, the first update occurs early ($t/T\approx0.2$), and most updates fall in the first half of the horizon, where value improves fastest, whereas under low SNR the first update is deferred to $t/T\in[0.39,0.75]$. ($iii$) \textit{The intervals are nonuniform}: on IST, updates concentrate at a few recurring times. Hence, the optimal meta-policy adapts both the timing and the risk of each update to local learning dynamics, which neither periodic nor equal-risk rules can replicate by construction.

\section{Conclusion}

We studied how to schedule the deployment of sequentially learned policies when each update may both improve or worsen performance. We formulated the problem as maximizing cumulative value under a budget on unsafe updates, reformulated it as a resource-constrained shortest-path problem, and developed an offline dynamic-programming solution using historical learning trajectories. Our asymptotic analysis characterizes the optimal update spacing, per-update risk, regret, and number of updates in terms of the local signal-to-noise ratio, revealing that stricter safety can be achieved with only a logarithmic increase in waiting time. Experiments on synthetic environments and real clinical data support these structural insights and demonstrate improved safety–value tradeoffs over benchmark update rules. Our results provide a principled way to coordinate when learned policies should be deployed to safely maximize performance gain.

Our framework also leaves several directions for future work. First, the structural analysis adopts a reduced-form diffusion model for the evolution of policy value. Extending the theory to broader learning dynamics, including non-Gaussian update risks derived directly from specific learning algorithms, would clarify the generality of the resulting scheduling laws. Second, the implemented scheduler relies on estimated rewards and unsafe-update probabilities, so translating empirical feasibility into finite-sample population guarantees requires additional uncertainty quantification. Finally, our safety criterion controls the frequency of harmful updates but not their magnitude. Incorporating severity-sensitive or tail-risk notions of deterioration would provide a richer notion of deployment safety.

\bibliographystyle{plainnat}
\bibliography{ref}

\clearpage
\appendix
\onecolumn
\section*{Appendices}

\numberwithin{equation}{section}
\numberwithin{lemma}{section}
\numberwithin{proposition}{section}
\numberwithin{theorem}{section}
\numberwithin{corollary}{section}
\numberwithin{assumption}{section}
\numberwithin{remark}{section}
\allowdisplaybreaks[2]

\section{Discussion: Problem setting extension}
\label{app:generalize}

We use binary treatment assignment to introduce the deployment problem, but the scheduling formulation does not depend on a causal model. More generally, let a learning procedure produce candidates $f_1,\ldots,f_T$ from an exogenous training stream, and let $V(f)$ be an integrable scalar measure of population performance, with larger values preferred. For example, for prediction, we can take $V(f)=-\mathbb E[\ell(f;Z)]$; for a decision rule, we can use its expected utility. We define an unsafe update by $V(f_j)<V(f_i)$ and retain the objective and expected-count constraint in \eqref{eq:obj}. We also assume that the schedule changes which candidate is deployed, not the distribution of the training sequence. As in the main paper, we choose the schedule before observing the future training trajectory (\ie, an offline setting).

Under the same mean-improvement condition as \Cref{asm:improve}, the newest-candidate reduction, path representation, and dynamic program apply after replacing $\pi_i$ by $f_i$. The asymptotic analysis likewise uses only the induced scalar value process and its stated regularity and tail conditions. We therefore obtain the same conclusions for another learning task when that task satisfies those conditions; the existence of a scalar performance criterion alone does not imply the diffusion assumption. For the statistical results below, we additionally require the independent historical trajectories and bounded values specified there. A positive affine change of the performance scale preserves every unsafe-update event and simply rescales rewards.

\section{Additional details for Section~\ref{sec:reform}}
\label{app:algorithm-proofs}

We write $\boldsymbol{R}_{ij}$ and $\boldsymbol{S}_{ij}$ for the entries of the reward and safety matrices, $V_i=V(\pi_i)$, $\nu_i=\mathbb E[V_i]$, and $p_{ij}=\mathbb P(V_j<V_i)$ for $1\leq i<j\leq T$. We set $p_{i,T+1}=0$. A newest-policy schedule has times
\[
1=t_0<t_1<\cdots<t_K\leq T<t_{K+1}=T+1,
\qquad h_k=t_{k+1}-t_k\quad(0\leq k\leq K).
\]
There are $K$ actual updates and $K+1$ holding intervals. In particular, the last edge contributes reward but no update risk. We denote the population objective and risk of a path $P$ by $J_T(P)$ and $S_T(P)$, respectively.

\subsection{\texorpdfstring{Proof of \Cref{lem:ss}}{Proof of the newest-policy reduction}}
\label{app:newest-policy}

\begin{proof}
The deterministic meta-policy class is finite, and its no-update member is feasible, so a maximizer exists. We show how to transform any feasible meta-policy into a newest-policy schedule without decreasing its objective or increasing its risk.

First, remove backward switches. Suppose a schedule switches from candidate $i$ to candidate $j<i$, and its next switch, if any, is from $j$ to $\ell$. We delete the switch to $j$ and retain $i$ until that next switch. On the intervening periods, the expected value weakly increases because $\nu_i\geq\nu_j$. For every realized candidate-value vector,
\begin{equation}
\label{eq:app-triangle}
\mathbbm 1\{V_\ell<V_i\}
\leq \mathbbm 1\{V_j<V_i\}+\mathbbm 1\{V_\ell<V_j\}.
\end{equation}
Indeed, if both events on the right fail, then $V_\ell\geq V_j\geq V_i$. Thus the new direct transition has risk no larger than the sum of the two removed risks. If there is no later switch, deleting the last switch removes a nonnegative risk. A resulting transition from $i$ to itself is a hold and has zero risk. Each operation removes at least one switch, so repetition terminates with a schedule whose deployed candidate indices are strictly increasing.

Write these indices as $1=j_0<j_1<\cdots<j_K$, and let $s_k\geq j_k$ be their original deployment times. We now advance the deployment of candidate $j_k$ to time $j_k$, processing the switches in increasing order. The preceding switch has already been advanced to $j_{k-1}<j_k$, and the next switch remains later than $s_k\geq j_k$, so this change preserves the order and availability constraints. It replaces the preceding, lower-index candidate by $j_k$ on additional periods, weakly increasing expected value. The ordered pairs of candidates compared at actual switches do not change, so neither does the total risk. The final schedule either retains its incumbent or deploys the candidate trained at the current time. Applying this transformation to a maximizer proves the claim.
\end{proof}

\subsection{\texorpdfstring{Proof of \Cref{thm:knapsack}}{Proof of the path representation}}
\label{app:path-proof}

\begin{proof}
Given a newest-policy schedule, form the path $P=((t_0,t_1),\ldots,(t_K,t_{K+1}))$ from $1$ to $T+1$. On the edge $(t_k,t_{k+1})$, we deploy $\pi_{t_k}$ for exactly $h_k$ periods. Therefore
\begin{equation}
\label{eq:app-path-identities}
J_T(P)=\frac1T\sum_{k=0}^K h_k\nu_{t_k}
      =\sum_{(i,j)\in P}\boldsymbol{R}_{ij},
\qquad
S_T(P)=\sum_{k=0}^{K-1}p_{t_k,t_{k+1}}
      =\sum_{(i,j)\in P}\boldsymbol{S}_{ij}.
\end{equation}
The second equality follows from linearity of expectation applied to the unsafe-update indicators.

Conversely, every source--sink path in the time-ordered graph has strictly increasing vertices and defines a valid newest-policy schedule by retaining its current vertex's candidate until the next vertex. The terminal edge retains the last candidate through period $T$. The two identities in \eqref{eq:app-path-identities} again hold. Thus paths and newest-policy schedules have the same rewards and risks, and \Cref{lem:ss} establishes equivalence with \eqref{eq:obj}. The direct path $(1,T+1)$ represents no updates. For a zero-length subpath, we use $\mathcal P_{ii}=\{\varnothing\}$, with zero reward and zero resource consumption.
\end{proof}

\subsection{Correctness and complexity of the dynamic program}
\label{app:dp}

\begin{proposition}[Exact solution of the rounded empirical problem]
\label{prop:app-dp}
With $B_\Delta=\lfloor\epsilon/\Delta\rfloor$ and $c_{ij}=\lceil\widehat{\boldsymbol{S}}_{ij}/\Delta\rceil$, where $c_{i,T+1}=0$, \Cref{alg:main} maximizes the estimated objective over paths satisfying $\sum_{(i,j)\in P}c_{ij}\leq B_\Delta$. Given the edge estimates, its time and memory requirements are
\[
O\!\left(T^2(B_\Delta+1)\right),
\qquad
O\!\left(T^2+T(B_\Delta+1)\right),
\]
respectively, when the edge matrices and backtracking information are stored.
\end{proposition}

\begin{proof}
We prove by induction in $j$ that $F(j,b)$ is the largest estimated reward of a path from $1$ to $j$ with rounded cost at most $b$. The initialization $F(1,b)=0$ gives the assertion at the source. A path reaching $j>1$ has a unique last edge $(i,j)$, where $i<j$. Removing that edge leaves a path with available budget $b-c_{ij}$. By the induction hypothesis, its reward is at most $F(i,b-c_{ij})$. Conversely, a path attaining a finite predecessor state followed by $(i,j)$ is a valid path to $j$ with the prescribed budget. Taking the maximum over predecessors gives the Bellman recursion in the main paper. Induction proves the state interpretation, and backtracking from $(T+1,B_\Delta)$ recovers an optimizer.

There are $(T+1)(B_\Delta+1)$ states, with at most $T$ predecessor comparisons per state. The state values and predecessor pointers require $O(T(B_\Delta+1))$ memory, and the edge matrices require $O(T^2)$ memory. Forming all pairwise empirical risks from $N$ value trajectories separately costs $O(NT^2)$ time. The source--sink edge has zero cost, so the rounded problem remains feasible even when $B_\Delta=0$.
\end{proof}

\subsection{Choosing the budget resolution}
\label{app:perturbation-proof}

Let $\mathcal M_{T,\bar K}$ be the newest-policy schedules with at most $\bar K$ actual updates, and define
\begin{equation}
\label{eq:app-frontier-definition}
J^\star_{T,\bar K}(b)
=\max\{J_T(P):P\in\mathcal M_{T,\bar K},\ S_T(P)\leq b\},
\qquad b\geq0.
\end{equation}
Without a count restriction, we take $\bar K=T-1$. With a restriction, we append a count coordinate to the dynamic-programming state and increase it only when $j\leq T$. Write $\widehat\nu_i=N^{-1}\sum_n\widehat V_i^{(n)}$, so that $\widehat{\boldsymbol{R}}_{ij}=(j-i)\widehat\nu_i/T$.

\begin{proposition}[Estimation and gridding error]
\label{prop:app-grid-error}
Suppose, on an event $\mathcal A$, that
\[
\max_i|\widehat\nu_i-\nu_i|\leq r,
\qquad
\max_{i<j\leq T}|\widehat{\boldsymbol{S}}_{ij}-p_{ij}|\leq s.
\]
Let $\widehat P_\Delta$ solve the rounded empirical problem within $\mathcal M_{T,\bar K}$, and put
\[
g_\Delta=\epsilon-\Delta\lfloor\epsilon/\Delta\rfloor,
\qquad
b_\Delta=\bar K(s+\Delta)+g_\Delta.
\]
On $\mathcal A$, we have $S_T(\widehat P_\Delta)\leq\epsilon+\bar K s$. If $b_\Delta\leq\epsilon$, then
\begin{align}
\label{eq:app-grid-sandwich}
J^\star_{T,\bar K}(\epsilon-b_\Delta)-r
&\leq\widehat J_T(\widehat P_\Delta)
\leq J^\star_{T,\bar K}(\epsilon+\bar K s)+r,\\
\label{eq:app-grid-oracle}
J_T(\widehat P_\Delta)
&\geq J^\star_{T,\bar K}(\epsilon-b_\Delta)-2r.
\end{align}
In particular, for $\omega_\epsilon(b)=J^\star_{T,\bar K}(\epsilon)-J^\star_{T,\bar K}(\epsilon-b)$,
\begin{equation}
\label{eq:app-grid-modulus}
J^\star_{T,\bar K}(\epsilon)-J_T(\widehat P_\Delta)
\leq 2r+\omega_\epsilon(b_\Delta).
\end{equation}
\end{proposition}

\begin{proof}
For every schedule, its holding fractions are nonnegative and sum to one. Hence, simultaneously over all paths,
\begin{equation}
\label{eq:app-reward-uniform}
|\widehat J_T(P)-J_T(P)|
\leq \sum_{(i,j)\in P}\frac{j-i}{T}|\widehat\nu_i-\nu_i|
\leq r.
\end{equation}
There is no factor $\bar K$ in this reward bound. Rounded feasibility implies
\[
S_T(\widehat P_\Delta)
\leq\sum_{e\in\widehat P_\Delta}\widehat{\boldsymbol{S}}_e+\bar K s
\leq\Delta B_\Delta+\bar K s
\leq\epsilon+\bar K s.
\]
For a comparator $P$ with $S_T(P)\leq\epsilon-b_\Delta$, upward rounding on its actual transitions gives
\[
\Delta\sum_{e\in P}c_e
\leq S_T(P)+\bar K s+\bar K\Delta
\leq\epsilon-g_\Delta=\Delta B_\Delta.
\]
The terminal edge contributes no rounding error. Thus every such population comparator is feasible for the rounded empirical problem. Empirical optimality and \eqref{eq:app-reward-uniform} give the left inequality in \eqref{eq:app-grid-sandwich} and \eqref{eq:app-grid-oracle}. The risk bound for the selected path gives the right inequality. Subtracting \eqref{eq:app-grid-oracle} from the population frontier at $\epsilon$ proves \eqref{eq:app-grid-modulus}.
\end{proof}

To derive an explicit dependence on $N$ and $T$, suppose the historical value vectors are independent copies of the future value vector, their coordinates lie in $[0,1]$, and their population values are observed without evaluation error. With $M=T(T-1)/2$ and $\delta\in(0,1)$, the bounded-mean inequality proved below and a union bound give
\begin{equation}
\label{eq:app-grid-hoeffding}
r=s=\sqrt{\frac{\log(2(T+M)/\delta)}{2N}},
\qquad \mathbb P(\mathcal A)\geq1-\delta.
\end{equation}
Consequently the direct reward error is $O(\sqrt{\log(T/\delta)/N})$, whereas the safety-budget displacement is
\[
O\!\left(\bar K\sqrt{\frac{\log(T/\delta)}{N}}+\bar K\Delta+\Delta\right).
\]
With no count cap, $\bar K=T-1$. No assumption about independence of edges within one trajectory is needed. For estimated population values, we use the evaluation result in \Cref{prop:app-evaluation} rather than interpreting noisy comparisons as observed true signs.

A budget displacement cannot generally be replaced by a uniform Lipschitz error in the objective. For example, let $T=2$, $V_1=1/2$, and let $V_2$ equal $1$ with probability $1-p$ and $0$ with probability $p$, where $0<p<1/2$. The frontier jumps from $1/2$ to $3/4-p/2$ at budget $p$. At $\epsilon=p$, any grid for which $p/\Delta$ is not an integer excludes the update, even along a sequence $\Delta\downarrow0$. Thus \eqref{eq:app-grid-modulus}, rather than an unconditional $O(\Delta)$ value-loss claim, is the appropriate guarantee. If an optimizer has safety slack at least $b_\Delta$, the modulus is zero; a Lipschitz frontier bound, when separately available on the relevant interval, gives the corresponding linear bound.

\paragraph{A theoretical resolution criterion.} For an allocated gridding tolerance $\kappa>0$ in risk units, we can choose $\Delta\leq\kappa/(\bar K+1)$, because $\bar K\Delta+g_\Delta\leq(\bar K+1)\Delta$. We should not make the grid substantially finer than the statistical accuracy unless a near-binding comparator makes that worthwhile. In the plug-in bound \eqref{eq:app-grid-hoeffding}, balancing $\bar K\Delta$ against $\bar K s$ suggests $\Delta$ of order $s$ or smaller. For the uncertainty-aware bound below, the analogous balance is $\Delta$ of order $\sqrt{\epsilon a_N/\bar K}+a_N$ when $\bar K\geq1$. Smaller $\Delta$ improves the budget approximation but increases the dynamic-programming cost through $B_\Delta\asymp\epsilon/\Delta$.

\paragraph{An empirical refinement procedure.} We can begin with $\Delta=\epsilon/B$ for a modest positive integer $B$, then double $B$ until the selected reward and schedule stabilize at a prespecified tolerance or the computational limit is reached. These nested grids have no residual $g_\Delta$, and every previously feasible path remains feasible after refinement because $\lceil2x\rceil\leq2\lceil x\rceil$. We record both the rounded and unrounded resource consumption at each refinement. Numerical stability is an empirical stopping rule, not a population-safety certificate; independent evaluation or the simultaneous bounds below provide that certificate. When $\epsilon=0$, we instead retain only edges with exactly zero chosen risk cost and solve the resulting acyclic path problem.

\subsection{An exact labeling alternative}
\label{app:labeling}

We can solve the unrounded path problem by propagating nondominated labels. Let $a_{ij}\geq0$ be the chosen edge resource, either a population risk, an empirical risk, or a conservative upper confidence cost, with $a_{i,T+1}=0$. A label at node $j$ stores its cumulative resource $s$, cumulative reward $r$, and a predecessor pointer. At the same node, $(s,r)$ dominates $(s',r')$ when $s\leq s'$ and $r\geq r'$, with at least one strict inequality. We retain only one representative when both coordinates agree.

\begin{algorithm}[H]
\caption{Exact nondominated-label algorithm}
\label{alg:app-labeling}
\begin{algorithmic}[1]
\REQUIRE Edge rewards $\widehat{\boldsymbol{R}}_{ij}$, nonnegative resources $a_{ij}$, budget $\epsilon$.
\STATE Set $\mathcal L_1=\{(0,0,\mathrm{null})\}$ and $\mathcal L_j=\varnothing$ for $j>1$.
\FOR{$j=2,\ldots,T+1$}
  \STATE Form labels $(s+a_{ij},r+\widehat{\boldsymbol{R}}_{ij},\mathrm{pointer})$ from every $i<j$ and every $(s,r,\cdot)\in\mathcal L_i$.
  \STATE Discard labels with resource greater than $\epsilon$.
  \STATE Sort the remaining labels by increasing resource and discard dominated labels; store the survivors in $\mathcal L_j$.
\ENDFOR
\STATE Return the highest-reward label in $\mathcal L_{T+1}$ and recover its path by backtracking.
\end{algorithmic}
\end{algorithm}

\begin{proposition}[Exactness of labeling]
\label{prop:app-labeling}
\Cref{alg:app-labeling} returns an exact maximizer of the unrounded problem for its chosen edge quantities. If $A_j$ labels are generated at node $j$ before pruning and $L_j$ survive, an implementation using sorting requires
\[
O\!\left(T^2+\sum_{j=2}^{T+1}A_j\log(A_j+1)\right)
\]
time and $O(T^2+\sum_jL_j)$ persistent storage, in addition to temporary storage for the labels being processed. The numbers of labels can be exponential in the size of a general resource-constrained acyclic path instance.
\end{proposition}

\begin{proof}
Any suffix from a node adds the same resource and reward to two labels at that node. Extending the dominating label therefore yields no larger resource and no smaller reward. Discarding a dominated prefix cannot remove the sole representative of a better feasible complete path. Nonnegative resources likewise justify discarding an over-budget prefix. Induction over the node order now shows that every feasible prefix is represented or dominated by a retained label. At the sink, selecting the largest reward is therefore exact.

Generating the incoming labels costs $O(A_j)$ operations. After sorting by resource, we keep a label only when its reward exceeds the best reward at a no-larger resource; equal-resource ties are resolved by retaining the largest reward. This costs $O(A_j\log(A_j+1))$ time. The stated storage follows from the retained labels and their pointers. To see the worst-case obstruction, consider a chain of binary-choice gadgets with resource and reward increments both equal to $2^{-k}$ for taking choice $k$, and zero for skipping it. Under a sufficiently large budget, every subset generates a different pair $(s,s)$. None dominates another, so there are exponentially many nondominated labels.
\end{proof}

Labeling avoids budget rounding and can be attractive when pruning leaves a small frontier. In contrast, the grid dynamic program bounds the resource dimension in advance and has the predictable pseudo-polynomial cost in \Cref{prop:app-dp}; this worst-case computational distinction motivates our main-paper choice. Exactness here refers to the specified real edge quantities, so floating-point tolerances must not be used to prune labels whose dominance is unresolved. With an update-count cap, we compare labels only at equal node and count, or include count in the dominance relation.

\subsection{An uncertainty-aware extension}
\label{app:uncertainty}

The plug-in costs in \Cref{alg:main} can underestimate population risk. We can instead assign an upper confidence cost to each actual edge and apply exactly the same path solver. The companions establish this construction for a fixed number of independent historical trajectories. We obtain an \emph{anytime-valid} version by allocating the failure probability across sample sizes. Here ``anytime'' refers to the number of historical trajectories used to plan the schedule, not to adapting deployment decisions to the realized future training stream.

\begin{assumption}[Historical trajectories for certification]
\label{ass:app-historical}
The vectors $(V_1^{(n)},\ldots,V_T^{(n)})$, $n\geq1$, are independent and identically distributed copies of the future candidate-value vector and are independent of that future vector. Every coordinate lies in $[0,1]$. Dependence between coordinates within a vector is unrestricted.
\end{assumption}

This bounded-value assumption is used only for the statistical guarantees, separately from the Gaussian value model used for the structural asymptotics. Note that independent resamples of a single observed dataset are not, merely by being generated separately, independent population trajectories satisfying \Cref{ass:app-historical}.

\begin{lemma}[Concentration and an empirical upper confidence cost]
\label{lem:app-concentration}
For $N$ independent $[0,1]$ variables with common mean $v$ and average $\bar X$, and $N$ independent Bernoulli variables with mean $p$ and average $\bar B$, we have, for $\ell>0$ and $a=\ell/N$,
\begin{equation}
\label{eq:app-concentration}
\mathbb P\!\left(|\bar X-v|>\sqrt{\frac{\ell}{2N}}\right)\leq2e^{-\ell},
\qquad
\mathbb P\!\left(|\bar B-p|>\sqrt{2pa}+\frac{2a}{3}\right)\leq2e^{-\ell}.
\end{equation}
The function $U_a(x)=\min\{1,x+\sqrt{2ax}+4a\}$ is nondecreasing on $[0,1]$. Whenever $|x-p|\leq\sqrt{2pa}+2a/3$, it satisfies
\begin{equation}
\label{eq:app-U-sandwich}
p\leq U_a(x)\leq p+2\sqrt{2pa}+7a.
\end{equation}
\end{lemma}

\begin{proof}
For a variable $X\in[0,1]$, the second derivative of $k(\lambda)=\log\mathbb E[e^{\lambda X}]$ is its variance under exponential tilting, which is at most $1/4$. Integrating twice gives $\log\mathbb E[e^{\lambda(X-v)}]\leq\lambda^2/8$. Independence and exponential Markov bounds, optimized over positive and negative $\lambda$, give the first inequality in \eqref{eq:app-concentration}.

For a Bernoulli variable $B$, we have
\[
\log\mathbb E[e^{\lambda(B-p)}]
\leq p(e^\lambda-1-\lambda)
\leq\frac{p\lambda^2}{2(1-|\lambda|/3)},\qquad |\lambda|<3.
\]
The last bound follows by expanding the exponential and using $k!\geq2\cdot3^{k-2}$ for $k\geq2$; the negative-argument bound follows as well. For $p>0$, optimization with $\lambda=t/(p+t/3)$ yields
\[
\mathbb P(|\bar B-p|>t)\leq2\exp\!\left(-\frac{Nt^2}{2(p+t/3)}\right).
\]
At $t=\sqrt{2pa}+2a/3$, direct expansion gives $t^2\geq2a(p+t/3)$, proving the second inequality. If $p=0$, the assertion is immediate.

The map $U_a$ is nondecreasing. Solving $p-x\leq\sqrt{2pa}+2a/3$ as a quadratic inequality in $\sqrt p$ gives
\[
p\leq\left(\sqrt{a/2}+\sqrt{x+7a/6}\right)^2
=x+\frac{5a}{3}+\sqrt{2ax+\frac{7a^2}{3}}
\leq x+\sqrt{2ax}+4a.
\]
Clipping at one preserves the lower bound. Also, $x\leq p+\sqrt{2pa}+2a/3\leq(\sqrt p+\sqrt a)^2$, and hence
\[
U_a(x)-p\leq(x-p)+\sqrt{2ax}+4a
\leq2\sqrt{2pa}+(2/3+\sqrt2+4)a
\leq2\sqrt{2pa}+7a.
\]
This proves \eqref{eq:app-U-sandwich}.
\end{proof}

At historical sample size $N$, let $\widehat\nu_{i,N}$ and $\widehat p_{ij,N}$ be the empirical value means and strict-failure frequencies. For a total failure probability $\delta\in(0,1)$, define
\begin{equation}
\label{eq:app-anytime-costs}
\delta_N=\frac{\delta}{N(N+1)},\qquad
\ell_N=\log\frac{2(M+T)}{\delta_N},\qquad
a_N=\frac{\ell_N}{N},\qquad
r_N=\sqrt{\frac{\ell_N}{2N}},\qquad
U_{ij,N}=U_{a_N}(\widehat p_{ij,N}),
\end{equation}
with $U_{i,T+1,N}=0$. We round $U_{ij,N}$ upward instead of rounding $\widehat p_{ij,N}$, use the same estimated rewards, and solve the path problem at budget $\epsilon$. Denote the resulting schedule by $\widehat P_N$.

\begin{theorem}[Anytime population safety and oracle performance]
\label{thm:app-anytime}
Under \Cref{ass:app-historical}, define
\begin{equation}
\label{eq:app-anytime-budget}
b_N=2\sqrt{2\bar K\epsilon a_N}+7\bar K a_N+\bar K\Delta+g_\Delta.
\end{equation}
With probability at least $1-\delta$, simultaneously for every integer $N\geq1$,
\begin{equation}
\label{eq:app-anytime-guarantee}
S_T(\widehat P_N)\leq\epsilon,
\qquad
J_T(\widehat P_N)\geq J^\star_{T,\bar K}(\epsilon-b_N)-2r_N
\quad\text{whenever }b_N\leq\epsilon.
\end{equation}
The no-update schedule is feasible for every historical dataset. The same guarantees apply at any almost surely finite, data-dependent stopping sample size. Including the randomness of the historical data and the independent future trajectory, the probability of at least one unsafe future update is at most $\epsilon+\delta$.
\end{theorem}

\begin{proof}
At sample size $N$, apply \Cref{lem:app-concentration} to all $T$ candidate coordinates and all $M$ actual edge indicators. A union bound with \eqref{eq:app-anytime-costs} gives failure probability at most $\delta_N$. Since $\sum_{N\geq1}\delta_N=\delta$, a second union bound gives an event $\mathcal A_\infty$ of probability at least $1-\delta$ on which, for every $N$,
\begin{equation}
\label{eq:app-all-N-event}
\max_i|\widehat\nu_{i,N}-\nu_i|\leq r_N,
\qquad
p_e\leq U_{e,N}\leq p_e+2\sqrt{2p_ea_N}+7a_N
\quad(e\text{ an actual edge}).
\end{equation}
The estimates at different sample sizes need not be independent.

On this event, every rounded-feasible path satisfies
\[
S_T(P)\leq\sum_{e\in P}U_{e,N}
\leq\Delta\sum_{e\in P}\left\lceil U_{e,N}/\Delta\right\rceil
\leq\Delta B_\Delta\leq\epsilon.
\]
This is simultaneous over paths and therefore holds for $\widehat P_N$. For a comparator with at most $\bar K$ updates and true risk $S$, Cauchy--Schwarz gives
\begin{equation}
\label{eq:app-certification-cost}
\sum_{e\in P}U_{e,N}
\leq S+2\sqrt{2a_N}\sum_{e\in P,\ e\text{ actual}}\sqrt{p_e}+7\bar K a_N
\leq S+2\sqrt{2\bar K S a_N}+7\bar K a_N.
\end{equation}
If $S\leq\epsilon-b_N$, substituting $S\leq\epsilon$ in the square root and adding at most $\bar K\Delta$ for rounding proves empirical feasibility of the comparator. The terminal edge has zero cost throughout. The uniform reward bound \eqref{eq:app-reward-uniform}, with $r=r_N$, then gives
\[
J_T(\widehat P_N)\geq\widehat J_T(\widehat P_N)-r_N
\geq\widehat J_T(P^\star)-r_N
\geq J_T(P^\star)-2r_N,
\]
where $P^\star$ optimizes the reduced-budget population problem. This proves the oracle inequality. All statements hold on the same event for all $N$, so evaluating them at a stopping sample size introduces no additional failure probability.

Conditional on historical data in $\mathcal A_\infty$, the selected schedule is fixed relative to the future trajectory. The union bound for its unsafe events gives a conditional probability at most $S_T(\widehat P_N)\leq\epsilon$. On the complementary event, we bound this probability by one. Averaging yields at most $\epsilon+\delta$. The direct terminal edge proves feasibility on every dataset.
\end{proof}

For a prespecified $N$, we can replace $\delta_N$ by $\delta$ and recover the fixed-sample confidence construction in the companions. The anytime version adds only the logarithmic sample-size term in $\ell_N$. Its certification cost scales as
\[
O\!\left(\sqrt{\frac{\bar K\epsilon\log(TN/\delta)}{N}}
+\frac{\bar K\log(TN/\delta)}{N}+\bar K\Delta+\Delta\right),
\]
while the reward error remains independent of the number of updates outside the logarithm. These costs certify population risk; numerical gridding alone does not.

\begin{proposition}[Certification with estimated population values]
\label{prop:app-evaluation}
Fix $N$. Suppose \Cref{ass:app-historical} holds for the true values, and evaluation estimates $\widetilde V_i^{(n)}$ satisfy an event
\begin{equation}
\label{eq:app-eval-event}
\mathcal A_\eta=\left\{\max_{n\leq N,\,i\leq T}|\widetilde V_i^{(n)}-V_i^{(n)}|\leq\eta\right\},
\qquad \mathbb P(\mathcal A_\eta)\geq1-\delta_{\mathrm{eval}},
\end{equation}
for a specified deterministic $\eta\geq0$. Use the conservative labels
\[
C_{ij}^{(n)}=\mathbf1\{\widetilde V_j^{(n)}-\widetilde V_i^{(n)}\leq2\eta\}
\]
in the empirical upper confidence costs. In this proposition set
\[
\ell_N=\log\frac{2(2M+T)}{\delta_{\mathrm{stat}}},\qquad
a_N=\ell_N/N,\qquad r_N=\sqrt{\ell_N/(2N)},\qquad
\zeta_T(\eta)=\max_{i<j\leq T}\mathbb P(0\leq V_j-V_i\leq4\eta).
\]
With
\[
b_{N,\eta}=\bar K\zeta_T(\eta)
+2\sqrt{2\bar K\{\epsilon+\bar K\zeta_T(\eta)\}a_N}
+7\bar K a_N+\bar K\Delta+g_\Delta,
\]
the selected schedule satisfies, with probability at least $1-\delta_{\mathrm{stat}}-\delta_{\mathrm{eval}}$,
\[
S_T(\widehat P)\leq\epsilon,
\qquad
J_T(\widehat P)\geq J^\star_{T,\bar K}(\epsilon-b_{N,\eta})-2(r_N+\eta)
\quad\text{if }b_{N,\eta}\leq\epsilon.
\]
No independence between the evaluation errors is required beyond \eqref{eq:app-eval-event}.
\end{proposition}

\begin{proof}
For each actual edge, define the true labels $B_{ij}^{(n)}=\mathbf1\{V_j^{(n)}-V_i^{(n)}<0\}$ and $D_{ij}^{(n)}=\mathbf1\{V_j^{(n)}-V_i^{(n)}\leq4\eta\}$. On $\mathcal A_\eta$, the deterministic inequalities
\begin{equation}
\label{eq:app-label-brackets}
B_{ij}^{(n)}\leq C_{ij}^{(n)}\leq D_{ij}^{(n)}
\end{equation}
hold. Their true means satisfy $p_{ij}\leq p^+_{ij}\leq p_{ij}+\zeta_T(\eta)$. Apply \Cref{lem:app-concentration} to the $2M$ collections of true $B$ and $D$ labels and the $T$ true value coordinates. These collections are independent across historical replications even when the estimated labels share evaluation data. A union bound produces an event of probability at least $1-\delta_{\mathrm{stat}}$.

On its intersection with $\mathcal A_\eta$, monotonicity of $U_{a_N}$ and \eqref{eq:app-label-brackets} give
\[
p_e\leq U_{a_N}(\bar B_e)\leq U_{a_N}(\bar C_e)
\leq U_{a_N}(\bar D_e)
\leq p_e^++2\sqrt{2p_e^+a_N}+7a_N.
\]
Thus the implemented costs upper-bound every population risk. For a comparator of true risk $S$, summing the upper bounds yields
\[
\sum_{e\in P}U_{a_N}(\bar C_e)
\leq S+\bar K\zeta_T(\eta)
+2\sqrt{2\bar K\{S+\bar K\zeta_T(\eta)\}a_N}+7\bar K a_N.
\]
The definition of $b_{N,\eta}$ and upward rounding therefore make every comparator at budget $\epsilon-b_{N,\eta}$ empirically feasible. Each estimated empirical value mean is within $r_N+\eta$ of its population mean, so the reward comparison in \Cref{thm:app-anytime} gives the oracle bound. A union bound over the two events completes the proof.
\end{proof}

For example, if every candidate is evaluated by averaging a score in $[0,1]$ on the same independent validation sample of size $n_{\mathrm{val}}$, conditioning on all trained candidates and applying \Cref{lem:app-concentration} gives \eqref{eq:app-eval-event} with $\eta=\sqrt{\log(2NT/\delta_{\mathrm{eval}})/(2n_{\mathrm{val}})}$. An unbounded or weighted policy-value estimator requires its own valid simultaneous evaluation bound. A margin condition $\mathbb P(|V_j-V_i|\leq x)\leq C_Tx^\alpha$ implies $\zeta_T(\eta)\leq C_T(4\eta)^\alpha$; it controls conservatism, not the validity of the safety assertion itself. Time-uniform evaluation can be incorporated by allocating both statistical and evaluation failure probabilities across $N$ in the same way as \eqref{eq:app-anytime-costs}.

We include this uncertainty-aware construction as an extension rather than replacing the plug-in algorithm studied in the main paper. Developing tighter time-uniform costs, adaptive allocation of historical simulation effort across edges, and sharper guarantees for the resulting resource-constrained path problem would require further analysis. Those questions are especially relevant to uncertainty-aware RCSPP methods and fall outside the present paper's deployment-scheduling contribution.

\section{\texorpdfstring{Proofs for Section~\ref{sec:method}}{Proofs for the theoretical analysis}}
\label{app:analysis-proofs}
\label{app:kkt}

We derive the continuum objective and constraint, solve the allocation pointwise, and then connect its value to the original integer problem. The proofs retain the $O(1)$ spacing correction, the $\Theta(\cdot)$ per-update risk, the $O(T^{-1})$ regret remainder, and the leading update count stated in the main paper.

\subsection{\texorpdfstring{Discussion on \Cref{ass:ito}}{Discussion of the diffusion assumption}}
\label{app:ito}

We impose the structured value-process assumption to keep the main analysis separate from the details of a particular learner. Throughout the proofs, we interpret its smoothness and positivity requirements as
\begin{equation}
\label{eq:app-coeff-regularity}
\mu,\sigma\in C^1([0,1]),\qquad
0<\mu_-\leq\mu(u)\leq\mu_+<\infty,\qquad
0<\sigma_-\leq\sigma(u)\leq\sigma_+<\infty,
\end{equation}
with bounded first derivatives and constants independent of $T$. Strictly positive continuous functions on the compact interval have these positive lower bounds. This excludes a vanishing-drift or degenerate-volatility regime from the theorem being proved.

We can obtain the assumed process from primitive constructions, and we  give two explicit sufficient constructions.

\paragraph{Gaussian innovations with an affine performance map.} Let $v(\theta)=v_0+a^\top\theta$ for a fixed nonzero vector $a$, and consider
\begin{equation}
\label{eq:app-affine-learning}
\theta_{t+1}=\theta_t+\frac1T\{b_{T,t}+H_{T,t}\xi_{t+1}\},
\qquad \xi_{t+1}\overset{\mathrm{iid}}\sim N_d(0,I_d),
\end{equation}
where $b_{T,t}$ and $H_{T,t}$ are deterministic. Suppose their projections satisfy
\[
a^\top b_{T,t}=\int_t^{t+1}\mu(s/T)\,\d s,
\qquad
\|H_{T,t}^\top a\|^2=\int_t^{t+1}\sigma(s/T)^2\,\d s.
\]
The value increments are then independent Gaussian variables with exactly the means and variances of the stochastic integrals in \Cref{ass:ito}. Their joint law is the integer skeleton of that process; a Brownian interpolation supplies a representation on an enlarged probability space. We can also express \eqref{eq:app-affine-learning} as stochastic-gradient descent with step $1/T$ on the sample loss $\ell_{T,t}(\theta;\xi)=-\theta^\top(b_{T,t}+H_{T,t}\xi)$.

\paragraph{Gradient-type diffusions with deterministic projected coefficients.} Let $\bar g(u,\theta)$ and $C(u,\theta)$ be a sample gradient's mean and covariance, and suppose the parameter process is the well-defined, nonexplosive diffusion
\begin{equation}
\label{eq:app-gradient-diffusion}
\d\Theta_\tau=-\frac\gamma T\bar g(\tau/T,\Theta_\tau)\,\d\tau
+\frac\gamma T C(\tau/T,\Theta_\tau)^{1/2}\,\d B_\tau.
\end{equation}
Local Lipschitz coefficients with suitable growth bounds give standard sufficient existence conditions. Let $v\in C^2$, take $v(\Theta_0)=V_0$, and impose, on the reachable parameter set,
\begin{align}
\label{eq:app-projected-identities}
-\gamma\nabla v(\theta)^\top\bar g(u,\theta)
+\frac{\gamma^2}{2T}\operatorname{tr}\{C(u,\theta)\nabla^2v(\theta)\}
&=\mu(u),\\
\gamma^2\nabla v(\theta)^\top C(u,\theta)\nabla v(\theta)
&=\sigma(u)^2.\nonumber
\end{align}
It\^o's formula gives drift $\mu(\tau/T)/T$ for $v(\Theta_\tau)$. Its martingale part equals $T^{-1}\sigma(\tau/T)\,\d W_\tau$, where
\[
W_t=\int_0^t\frac{\gamma\nabla v(\Theta_\tau)^\top C(\tau/T,\Theta_\tau)^{1/2}}{\sigma(\tau/T)}\,\d B_\tau.
\]
The integrand has squared norm one by \eqref{eq:app-projected-identities}, so this continuous martingale has quadratic variation $t$ and is a standard Brownian motion. Integration proves \Cref{ass:ito} for the diffusion learner. The projection identities, jointly with differentiability, remove the random state from the scalar drift and volatility.

\subsection{\texorpdfstring{Proof of \Cref{lem:continuous}}{Proof of the continuum reduction}}
\label{app:continuous-proof}

\paragraph{Notation and asymptotic regime.} We index the integer periods by $0,\ldots,T-1$ in these proofs, relabeling the original candidate $i$ as $i-1$. Let $V_T(s)$ denote the corresponding continuous Gaussian value process, so $V_T(i-1)=V(\pi_i)$ at the original integer checkpoints. Its coefficients are initially $\mu(u+1/T)$ and $\sigma(u+1/T)$. We extend the positive $C^1$ functions in \eqref{eq:app-coeff-regularity} across the final auxiliary interval when necessary and suppress this shift in the calculations. The shifted coefficients differ uniformly by $O(T^{-1})$; the resulting changes in $B$ and $\int_0^1I(u)\,\d u$ are $O(T^{-1})$. They therefore change the displayed scaled regret by $O(L_T/T)=o(1)$ and do not change any leading count or risk order below.

A schedule has times $0=t_0<t_1<\cdots<t_K<T=t_{K+1}$, with $K$ actual updates and $K+1$ holds. Write $h_k=t_{k+1}-t_k$ and $u_k=t_k/T$. Only $k=0,\ldots,K-1$ contribute update risk; the last interval is a terminal hold. For a fixed budget, take $\epsilon_T=\epsilon>0$. More generally, all the estimates below are uniform in the regime
\begin{equation}
\label{eq:app-budget-regime}
T^{-a_0}\leq\epsilon_T\leq\bar\epsilon,
\qquad L=L_T=\log(T/\epsilon_T)\asymp\log T,
\qquad 0<a_0<\infty,\quad 0<\bar\epsilon<\infty.
\end{equation}
The coefficient functions and their positive bounds are fixed as $T$ grows. Constants in the $O(\cdot)$ and $\Theta(\cdot)$ terms may depend on those bounds and the budget regime, but not on a particular update.

For an integer schedule $P$, the exact regret and risk are
\begin{equation}
\label{eq:app-discrete-regret-risk}
\mathcal R_T(P)=\frac1T\sum_{k=0}^{K}\sum_{j=0}^{h_k-1}\mathbb E\big[V_T(t_k+j)-V_T(t_k)\big],
\qquad
S_T(P)=\sum_{k=0}^{K-1}\mathbb P\big\{V_T(t_{k+1})<V_T(t_k)\big\}.
\end{equation}
These are the quantities in the main problem after relabeling. To pass from sums to integrals, define the step functions
\begin{equation}
\label{eq:app-step-functions}
\tau_T(u)=u_k,\qquad H_T(u)=h_k,\qquad u\in[u_k,u_{k+1}).
\end{equation}
Here $H_T$ records a given partition, whereas $h_T$ will be the positive real-valued allocation optimized in \eqref{eq:continuous}.

Set
\begin{equation}
\label{eq:app-Q-f}
Q(x)=\Phi(-\sqrt{2x}),\qquad f(x)=\frac{Q(x)}x,\qquad I(u)=\frac{\mu(u)^2}{2\sigma(u)^2},\qquad A(u)=\frac{\mu(u)}{I(u)},\qquad B=\int_0^1A(u)\,\d u.
\end{equation}
The local probability and risk rate are
\begin{equation}
\label{eq:app-local-risk-rate}
q(u,h)=Q(I(u)h),\qquad r(u,h)=\frac{q(u,h)}h=I(u)f(I(u)h).
\end{equation}
For a budget $b>0$, let
\begin{equation}
\label{eq:app-Fc}
F_T^{\mathrm c}(b)=\inf_{h_T(\cdot)>0}\left\{\int_0^1\mu(u)h_T(u)\,\d u:\int_0^1r(u,h_T(u))\,\d u\leq\frac bT\right\},
\end{equation}
where the infimum is over positive measurable functions with finite displayed integrals. We first establish the estimates needed for the objective and constraint.

\begin{lemma}[Normal-tail estimates]
\label{lem:app-normal}
Let $c_0=(2\sqrt\pi)^{-1}$. The functions $Q$ and $f$ are strictly decreasing and strictly convex on $(0,\infty)$. As $x\to\infty$,
\begin{align}
\label{eq:app-normal-tail}
Q(x)&=c_0e^{-x}x^{-1/2}\{1+O(x^{-1})\},\\
f(x)&=c_0e^{-x}x^{-3/2}\{1+O(x^{-1})\},\qquad -f'(x)=c_0e^{-x}x^{-3/2}\{1+O(x^{-1})\}.\nonumber
\end{align}
Consequently $f(x)/(-f'(x))\to1$ and $f''(x)/(-f'(x))\to1$. Moreover, $-f'$ decreases continuously from $+\infty$ to zero.
\end{lemma}

\begin{proof}
For the standard normal density $\phi$, integration by parts gives
$$
1-\Phi(z)=\frac{\phi(z)}z-\int_z^\infty\frac{\phi(y)}{y^2}\,\d y,
\qquad
\frac{z\phi(z)}{1+z^2}\leq1-\Phi(z)\leq\frac{\phi(z)}z.
$$
Substituting $z=\sqrt{2x}$ proves the expansion for $Q$. Direct differentiation yields
$$
Q'(x)=-c_0e^{-x}x^{-1/2},\qquad Q''(x)=c_0e^{-x}\left(x^{-1/2}+\frac12x^{-3/2}\right),
$$
$$
f'(x)=\frac{Q'(x)}x-\frac{Q(x)}{x^2}<0,
\qquad
f''(x)=\frac{Q''(x)}x-\frac{2Q'(x)}{x^2}+\frac{2Q(x)}{x^3}>0.
$$
These identities prove the remaining expansions and ratios. Since $Q(x)\to1/2$ as $x\downarrow0$, we have $-f'(x)\to\infty$ there; \eqref{eq:app-normal-tail} gives its limit zero at infinity. Strict convexity gives the asserted monotonicity.
\end{proof}

\begin{lemma}[Local interval probabilities]
\label{lem:app-intervals}
For an interval $[s,s+h]\subseteq[0,T]$, define
$$
M_{s,h}=\int_s^{s+h}\mu(v/T)\,\d v,
\qquad
Q^{\mathrm{var}}_{s,h}=\int_s^{s+h}\sigma(v/T)^2\,\d v,
\qquad
\mathcal I_{s,h}=\frac{M_{s,h}^2}{2Q^{\mathrm{var}}_{s,h}}.
$$
Its exact unsafe-update probability is $q_T(s,h)=Q(\mathcal I_{s,h})$. Uniformly in admissible $s,h$,
\begin{equation}
\label{eq:app-local-expansions}
M_{s,h}=\mu(s/T)h+O(h^2/T),\qquad Q^{\mathrm{var}}_{s,h}=\sigma(s/T)^2h+O(h^2/T),\qquad \mathcal I_{s,h}=I(s/T)h+O(h^2/T).
\end{equation}
Also,
\begin{equation}
\label{eq:app-global-exponent}
\mathcal I_{s,h}\leq\int_s^{s+h}I(v/T)\,\d v\leq h\max_{v\in[s,s+h]}I(v/T).
\end{equation}
For fixed $0<a<d<\infty$ and $aL\leq h\leq dL$,
\begin{equation}
\label{eq:app-relative-local}
q_T(s,h)=q(s/T,h)\{1+O(L^2/T)\}.
\end{equation}
\end{lemma}

\begin{proof}
Under \Cref{ass:ito}, the interval increment is Gaussian with mean $M_{s,h}/T$ and variance $Q^{\mathrm{var}}_{s,h}/T^2$, proving the exact probability formula. Bounded first derivatives give the first two expansions in \eqref{eq:app-local-expansions}. Since $Q^{\mathrm{var}}_{s,h}\geq\sigma_-^2h$, division gives the exponent expansion. Cauchy--Schwarz gives
$$
M_{s,h}^2\leq\left(\int_s^{s+h}\sigma(v/T)^2\,\d v\right)\left(\int_s^{s+h}\frac{\mu(v/T)^2}{\sigma(v/T)^2}\,\d v\right),
$$
which proves \eqref{eq:app-global-exponent}. Finally, on $h\asymp L$, the two exponent arguments differ by $O(L^2/T)=o(1)$ and both tend to infinity. By \Cref{lem:app-normal}, $|(\log Q)'|$ is bounded on these arguments. Applying the mean value theorem to $\log Q$ proves \eqref{eq:app-relative-local}.
\end{proof}

\begin{lemma}[A competitive schedule and its mesh bound]
\label{lem:app-competitive}
There is a feasible integer schedule with regret $O(L/T)$. Every integer schedule satisfying $\mathcal R_T(P)\leq C_1L/T$, for a fixed $C_1$, obeys
\begin{equation}
\label{eq:app-mesh}
\sum_{k=0}^{K}h_k^2=O(TL),\qquad \max_k h_k=O(\sqrt{TL}),\qquad \max_k h_k/T\longrightarrow0.
\end{equation}
\end{lemma}

\begin{proof}
Let $I_{\min}=\min_{u\in[0,1]}I(u)>0$ and take full integer holds of length $\lceil L/I_{\min}\rceil$, followed by a terminal remainder. By \Cref{lem:app-intervals}, the exponent of every actual update is at least $L-O(L^2/T)$. The normal-tail upper bound gives probability at most $CL^{-1/2}e^{-L}$ per update. There are at most $CT/L$ updates, so their total risk is at most $C\epsilon_T/L^{3/2}\leq\epsilon_T$ for sufficiently large $T$. Bounded drift and holding lengths $O(L)$ give regret $O(L/T)$.

For any integer schedule, positivity of the drift gives $\mathbb E[V_T(s+j)-V_T(s)]\geq\mu_-j/T$. Hence
$$
\mathcal R_T(P)\geq\frac{\mu_-}{2T^2}\sum_{k=0}^{K}h_k(h_k-1).
$$
Since $\sum_kh_k=T$, the assumed regret bound yields $\sum_kh_k^2=O(TL)$. The maximum is at most the square root of this sum, which proves the remaining statements.
\end{proof}

\paragraph{(I) Objective.} Maximizing expected deployed value is equivalent to minimizing \eqref{eq:app-discrete-regret-risk}, because the always-update benchmark is independent of the schedule. Within a hold starting at $t_k$, we have
$$
\mathbb E[V_T(t_k+j)-V_T(t_k)]
=\frac1T\int_{t_k}^{t_k+j}\mu(v/T)\,\d v
=\frac{\mu(u_k)j}{T}+O(j^2/T^2).
$$
Summing first within each hold and then across holds gives
\begin{align}
\label{eq:app-objective-sum}
\mathcal R_T(P)
&=\frac1{T^2}\sum_{k=0}^{K}\left\{\mu(u_k)\sum_{j=0}^{h_k-1}j+O(h_k^3/T)\right\}\\
&=\frac1{2T^2}\sum_{k=0}^{K}\mu(u_k)h_k(h_k-1)+O\!\left(\frac{\sum_kh_k^3}{T^3}\right).\nonumber
\end{align}
For a schedule in \Cref{lem:app-competitive},
$$
\frac{\sum_kh_k^3}{T^2}
\leq\frac{\max_kh_k}{T^2}\sum_kh_k^2
=O(L^{3/2}/\sqrt T)=o(1),
\qquad
\frac1T\sum_k\mu(u_k)h_k\leq\mu_+.
$$
Using the step functions in \eqref{eq:app-step-functions}, we obtain
\begin{align}
\label{eq:app-regret-expansions}
2T\mathcal R_T(P)
&=\frac1T\sum_{k=0}^{K}\mu(u_k)h_k^2+O(1)\\
&=\int_0^1\mu(\tau_T(u))H_T(u)\,\d u+O(1)
=\int_0^1\mu(u)H_T(u)\,\d u+O(1).\nonumber
\end{align}
The last equality follows from
$$
\left|\int_0^1\big\{\mu(\tau_T(u))-\mu(u)\big\}H_T(u)\,\d u\right|
\leq\frac{C}{T^2}\sum_kh_k^3=o(1).
$$
Thus the integral objective represents $2T$ times regret up to a bounded error, uniformly over this class of schedules.

\paragraph{(II) Constraint.} For this local calculation, suppose every risk-bearing hold satisfies $aL\leq h_k\leq dL$ for fixed positive $a,d$. By \Cref{lem:app-intervals},
\begin{align}
\label{eq:app-risk-sum}
\frac{S_T(P)}T
&=\frac1T\sum_{k=0}^{K-1}q(u_k,h_k)\{1+O(L^2/T)\}\\
&=\left\{\sum_{k=0}^{K-1}\frac{h_k}{T}r(u_k,h_k)\right\}\{1+O(L^2/T)\}.\nonumber
\end{align}
On $[u_k,u_{k+1})$, the arguments $I(u)h_k$ and $I(u_k)h_k$ differ by $O(h_k^2/T)$. The logarithmic derivative bound used above therefore gives $r(u,h_k)=r(u_k,h_k)\{1+O(L^2/T)\}$. Consequently,
\begin{equation}
\label{eq:app-risk-integral}
\frac{S_T(P)}T=\left\{\int_0^{u_K}r(u,H_T(u))\,\d u\right\}\{1+O(L^2/T)\}.
\end{equation}
The integral stops at $u_K$ because the terminal hold has no update risk. In the construction below, we upper-bound this integral by the risk integral of a smooth allocation over $[0,1]$; no update is added at the terminal endpoint.

\paragraph{(III) Leading-order form.} Replacing the partition-dependent step function by a positive allocation $h_T$ gives \eqref{eq:app-Fc}, which is \eqref{eq:continuous} at $b=\epsilon_T$. The substitution $x(u)=I(u)h_T(u)$ makes the two integrands
$$
\mu(u)h_T(u)=A(u)x(u),\qquad r(u,h_T(u))=I(u)f(x(u)).
$$
The resulting optimization is
\begin{equation}
\label{eq:app-allocation-x}
F_T^{\mathrm c}(b)=\inf_{x(\cdot)>0}\left\{\int_0^1A(u)x(u)\,\d u:\int_0^1I(u)f(x(u))\,\d u\leq\frac bT\right\}.
\end{equation}
We solve this functional problem directly in the next subsection. The lower and upper bounds in \Cref{app:regret-count-proof} then establish the precise optimized comparison
\begin{equation}
\label{eq:app-leading-value}
2T\mathcal R_T^\star(\epsilon_T)=F_T^{\mathrm c}(\epsilon_T)+O(1).
\end{equation}
That comparison completes \Cref{lem:continuous}: its lower bound applies to arbitrary feasible integer schedules, and its upper bound constructs a feasible integer schedule.

\subsection{\texorpdfstring{Proof of \Cref{thm:update-interval}}{Proof of the update-interval expansion}}
\label{app:update-interval-proof}

\begin{proof}
We solve \eqref{eq:app-allocation-x} for $b\in[\epsilon_T/2,\epsilon_T]$ and then take $b=\epsilon_T$. \textcolor{black}{This argument concerns the functional problem itself and does not use the optimized comparison \eqref{eq:app-leading-value}.}

\paragraph{(I) Pointwise minimization.} The Lagrangian is
$$
\mathcal L(x,\lambda)=\int_0^1\big\{A(u)x(u)+\lambda I(u)f(x(u))\big\}\,\d u-\lambda\frac bT.
$$
For a fixed $\lambda>0$, the integrand contains no terms coupling different values of $u$, so we minimize it separately at each $u$. By \Cref{lem:app-normal}, it is strictly convex in $x$ and diverges as $x\downarrow0$ or $x\uparrow\infty$. Its unique minimizer $x_\lambda(u)$ satisfies
\begin{equation}
\label{eq:app-stationarity}
A(u)+\lambda I(u)f'(x_\lambda(u))=0
\quad\Longleftrightarrow\quad
-f'(x_\lambda(u))=\frac{\mu(u)}{\lambda I(u)^2}.
\end{equation}

\paragraph{(II) Binding constraint.} Since $-f'$ is a strictly decreasing bijection from $(0,\infty)$ onto $(0,\infty)$, the solution $x_\lambda(u)$ is continuous in $u$, strictly increasing in $\lambda$, and tends uniformly to zero or infinity as $\lambda$ tends to zero or infinity, respectively. The coefficient ratios are bounded above and away from zero. It follows that $\int_0^1I(u)f(x_\lambda(u))\,\d u$ decreases continuously from infinity to zero. There is therefore a unique $\lambda_b>0$ such that
\begin{equation}
\label{eq:app-binding-budget}
\int_0^1I(u)f(x_{\lambda_b}(u))\,\d u=\frac bT.
\end{equation}
Write $x_b=x_{\lambda_b}$. For every feasible $x$, pointwise optimality gives
$$
A(u)x(u)+\lambda_bI(u)f(x(u))\geq A(u)x_b(u)+\lambda_bI(u)f(x_b(u)).
$$
Integrating this inequality and using feasibility and \eqref{eq:app-binding-budget} yields $\int_0^1Ax\geq\int_0^1Ax_b$. Strict convexity gives uniqueness up to null sets. Thus $h_b=x_b/I$ is the continuous representative of the optimizer of \eqref{eq:app-Fc}.

\paragraph{(III) Multiplier and waiting-time rates.} As $b/T\to0$, the binding equation forces $\lambda_b\to\infty$, and hence $x_b(u)\to\infty$ uniformly. Substituting \eqref{eq:app-stationarity} into \eqref{eq:app-binding-budget} gives
\begin{align}
\label{eq:app-multiplier}
\frac bT
&=\int_0^1 I(u)\big[-f'(x_b(u))\big]\frac{f(x_b(u))}{-f'(x_b(u))}\,\d u\\
&=\frac1{\lambda_b}\int_0^1A(u)\frac{f(x_b(u))}{-f'(x_b(u))}\,\d u
=\frac B{\lambda_b}\{1+o(1)\}.\nonumber
\end{align}
Consequently $\lambda_b\sim BT/b$ and $\log\lambda_b=L_b+O(1)$, where $L_b=\log(T/b)=L+O(1)$. Taking logarithms in \eqref{eq:app-stationarity} and applying \eqref{eq:app-normal-tail} yields
\begin{equation}
\label{eq:app-log-root}
x_b(u)+\frac32\log x_b(u)
=\log\lambda_b+\log\!\left(\frac{c_0I(u)^2}{\mu(u)}\right)+O(x_b(u)^{-1})
=L_b+O(1).
\end{equation}
The error is uniform in $u$. First, this identity gives $x_b(u)=L_b+O(\log L_b)$: for large $x_b$, it bounds $x_b$ above by $L_b+C$, and substitution into its logarithm gives the corresponding lower bound. Therefore
$$
\log x_b(u)=\log L_b+O(\log L_b/L_b).
$$
Substituting this back into \eqref{eq:app-log-root} proves
\begin{equation}
\label{eq:app-optimizer-expansion}
x_b(u)=L_b-\frac32\log L_b+O(1),\qquad
h_b(u)=\frac1{I(u)}\left[L_b-\frac32\log L_b+O(1)\right].
\end{equation}
Taking $b=\epsilon_T$ proves \Cref{thm:update-interval}. In particular, $h_b\asymp L$ uniformly, and integration gives
\begin{equation}
\label{eq:app-Fc-expansion}
F_T^{\mathrm c}(b)=\int_0^1A(u)x_b(u)\,\d u
=B\left[L_b-\frac32\log L_b\right]+O(1).
\end{equation}
For later use, differentiate the logarithm of \eqref{eq:app-stationarity}:
$$
-\frac{f''(x_b(u))}{-f'(x_b(u))}x_b'(u)
=\frac{\mu'(u)}{\mu(u)}-2\frac{I'(u)}{I(u)}.
$$
By \Cref{lem:app-normal} and the coefficient bounds, $x_b'=O(1)$ and $h_b'=O(L)$, uniformly for $b\in[\epsilon_T/2,\epsilon_T]$.
\end{proof}

\subsection{\texorpdfstring{Proof of \Cref{cor:risk}}{Proof of the risk-allocation formula}}
\label{app:risk-proof}

\begin{proof}
Consider a full interval generated by $h_b$, for $b\in[\epsilon_T/2,\epsilon_T]$. Its length is $h_b(u_k)$ for the continuous construction, or $\lceil h_b(u_k)\rceil$ for its integer realization. We first evaluate the local probability before rounding. By \eqref{eq:app-stationarity} and \eqref{eq:app-multiplier},
\begin{align}
\label{eq:app-local-risk-order}
q(u,h_b(u))
&=Q(x_b(u))=x_b(u)f(x_b(u))\\
&=\frac{x_b(u)\mu(u)}{\lambda_bI(u)^2}\frac{f(x_b(u))}{-f'(x_b(u))}
=\Theta\!\left(\frac{\epsilon_T L}{TB}\frac{\mu(u)}{I(u)^2}\right).\nonumber
\end{align}
Here $x_b\asymp L$, $\lambda_b\asymp TB/b$, $b\asymp\epsilon_T$, and the ratio $f/(-f')$ is bounded above and away from zero for large $T$.

Write an integer length as $j=h_b(u)+d$ with $0\leq d<1$. Its local exponent changes from $x_b(u)$ to $x_b(u)+I(u)d$. The normal-tail formula gives
$$
\frac{Q(x_b(u)+I(u)d)}{Q(x_b(u))}
=e^{-I(u)d}\left(\frac{x_b(u)}{x_b(u)+I(u)d}\right)^{1/2}\{1+O(L^{-1})\}.
$$
This ratio is bounded above and away from zero uniformly in $u,d$. Since both full interval lengths are of order $L$, \eqref{eq:app-relative-local} then proves
\begin{equation}
\label{eq:app-per-update-risk}
\mathbb P\big\{V_T(t_{k+1})<V_T(t_k)\big\}
=\Theta\!\left(\frac{\epsilon_T L}{TB}\frac{\mu(u_k)}{I(u_k)^2}\right),\qquad k=0,\ldots,K-1.
\end{equation}
Reversing the index shift gives the main paper's policy-value notation for integer updates. The feasible integer schedule used in the regret bound below is the rounded construction with $b=\epsilon_T/2$, so it satisfies the same risk order.
\end{proof}

\subsection{\texorpdfstring{Proof of \Cref{cor:regret}}{Proof of the regret and count results}}
\label{app:regret-count-proof}

\begin{proof}
Let $z_L=L-\tfrac32\log L$ and $D=\int_0^1I(u)\,\d u$. We first match a lower bound over feasible integer schedules with a feasible construction. We then establish the leading count for an optimal integer schedule.

\paragraph{(I) Lower bound by weighted Jensen's inequality.} By \Cref{lem:app-competitive}, we only need to consider schedules with regret at most $C_1L/T$, for a sufficiently large fixed $C_1$. For each risk-bearing hold, define
$$
\mu_k^- =\min_{v\in[t_k,t_{k+1}]}\mu(v/T),\qquad
I_k^+ =\max_{v\in[t_k,t_{k+1}]}I(v/T),\qquad
x_k=I_k^+h_k,\qquad
w_k=\frac{h_k}{T}\frac{\mu_k^-}{I_k^+}.
$$
Put $W_T=\sum_{k<K}w_k$ and $Z_T=\sum_{k<K}w_kx_k$. Coefficient smoothness and the mesh bound imply
\begin{equation}
\label{eq:app-weight-total}
W_T=B+O\!\left(\frac{\sum_kh_k^2}{T^2}+\frac{h_K}{T}\right)
=B+O\!\left(\sqrt{L/T}\right).
\end{equation}
Indeed, $\mu_k^-/I_k^+$ differs from $A(u)$ by at most $Ch_k/T$ within its interval, and the omitted terminal interval has width $h_K/T$. In particular, $W_T$ is positive and bounded away from zero for large $T$.

The exact Gaussian exponent bound \eqref{eq:app-global-exponent} gives $q_T(t_k,h_k)\geq Q(x_k)$. Let $I_- =\min I>0$ and $d_0=I_-^2/\mu_+$. Then
\begin{align}
\label{eq:app-weighted-risk}
S_T(P)
&\geq\sum_{k<K}Q(x_k)
=T\sum_{k<K}\frac{h_k}{T}I_k^+f(x_k)\\
&=T\sum_{k<K}w_k\frac{(I_k^+)^2}{\mu_k^-}f(x_k)
\geq d_0T W_T f\!\left(\frac{Z_T}{W_T}\right).\nonumber
\end{align}
The last inequality uses convexity of $f$ and Jensen's inequality with weights $w_k/W_T$. For a feasible schedule, it follows that
$$
f\!\left(\frac{Z_T}{W_T}\right)\leq\frac{\epsilon_T}{d_0TW_T}.
$$
Let $y_T$ solve $f(y_T)=\epsilon_T/(d_0TW_T)$. The solution is unique because $f$ decreases from infinity to zero. Since $W_T$ stays between fixed positive bounds, \eqref{eq:app-normal-tail} gives
$$
y_T+\frac32\log y_T=L+O(1),\qquad y_T=L-\frac32\log L+O(1).
$$
Monotonicity of $f$ therefore yields $Z_T\geq W_T\{z_L+O(1)\}=Bz_L+O(1)$; here \eqref{eq:app-weight-total} gives $(W_T-B)L=o(1)$.

Within an integer hold, the mean increment is at least $\mu_k^-j/T$. Including the terminal cost, which is nonnegative, we obtain
$$
2T\mathcal R_T(P)
\geq\frac1T\sum_{k<K}\mu_k^-h_k(h_k-1)
\geq Z_T-\mu_+.
$$
Consequently,
\begin{equation}
\label{eq:app-global-lower}
2T\mathcal R_T^\star(\epsilon_T)\geq B\left[L-\frac32\log L\right]-O(1).
\end{equation}
Schedules outside the competitive class have a larger regret than this bound for sufficiently large $C_1$.

\paragraph{(II) A feasible integer upper bound.} Take $b=\epsilon_T/2$ and use the continuum optimizer $h_b$ from \Cref{app:update-interval-proof}. Starting at $t_0=0$, propose the integer length
\begin{equation}
\label{eq:app-clock-construction}
j_k=\lceil h_b(t_k/T)\rceil,\qquad t_{k+1}=t_k+j_k\quad\text{when }t_k+j_k<T.
\end{equation}
When the proposed endpoint is at least $T$, retain the incumbent to $T$. Thus every full hold has length $\Theta(L)$, and the terminal hold is at most $CL$. Denote this schedule by $\widehat P_T$ and its step function by $H_T$.

We verify feasibility directly. Since $h_b'=O(L)$ and every hold has width $O(L/T)$ in normalized time, on each full interval
$$
H_T(u)=\lceil h_b(u_k)\rceil\geq h_b(u)-CL^2/T.
$$
The rate $r(u,h)$ decreases in $h$, and for $h\asymp L$ its logarithmic derivative in $h$ is bounded:
$$
\frac{\partial}{\partial h}\log r(u,h)=I(u)\frac{Q'(I(u)h)}{Q(I(u)h)}-\frac1h=O(1).
$$
It follows that $r(u,H_T(u))\leq r(u,h_b(u))\{1+O(L^2/T)\}$ on the full intervals. Using \eqref{eq:app-risk-integral} and the binding budget of $h_b$ gives
\begin{equation}
\label{eq:app-attaining-risk}
S_T(\widehat P_T)
\leq T\{1+O(L^2/T)\}\int_0^{u_K}r(u,h_b(u))\,\d u
\leq b\{1+O(L^2/T)\}\leq\epsilon_T
\end{equation}
for sufficiently large $T$. The fixed half-budget leaves room for both the local approximation and rounding errors.

For the objective, every hold, including the terminal remainder, satisfies
$$
H_T(u)\leq h_b(u)+1+CL^2/T.
$$
The schedule has regret $O(L/T)$ because its holds are $O(L)$, so \eqref{eq:app-regret-expansions} applies. Hence
\begin{align}
\label{eq:app-integer-upper}
2T\mathcal R_T(\widehat P_T)
&=\int_0^1\mu(u)H_T(u)\,\d u+O(1)\\
&\leq F_T^{\mathrm c}(\epsilon_T/2)+O(1)
=B\left[L-\frac32\log L\right]+O(1).\nonumber
\end{align}
The last equality uses \eqref{eq:app-Fc-expansion} and $\log(T/b)=L+\log2$. Combining \eqref{eq:app-global-lower} and \eqref{eq:app-integer-upper} proves
\begin{equation}
\label{eq:app-regret-result}
\mathcal R_T^\star(\epsilon_T)=\frac B{2T}\left[L_T-\frac32\log L_T\right]+O(T^{-1}).
\end{equation}
Together with \eqref{eq:app-Fc-expansion} at $b=\epsilon_T$, this also proves \eqref{eq:app-leading-value} and completes the optimized continuum reduction in \Cref{lem:continuous}. The constructed schedule satisfies $0\leq\mathcal R_T(\widehat P_T)-\mathcal R_T^\star(\epsilon_T)=O(T^{-1})$ and, on every risk-bearing interval,
$$
I(u_k)h_k=L-\frac32\log L+O(1).
$$

\paragraph{(III) Counting the constructed schedule.} For a full hold of \eqref{eq:app-clock-construction}, \eqref{eq:app-optimizer-expansion} gives $h_k=L\{1+o(1)\}/I(u_k)$ uniformly. Therefore
$$
\frac L T K(\widehat P_T)
=\sum_{k<K}\frac{h_k}{T}\frac L{h_k}
=\sum_{k<K}\frac{h_k}{T}I(u_k)\{1+o(1)\}
\longrightarrow\int_0^1I(u)\,\d u=D.
$$
The last step is a Riemann sum, since the mesh is $O(L/T)$ and the terminal interval has vanishing normalized length.

\paragraph{(IV) Counting an optimal schedule.} Let $P_T^\star$ be any optimal integer schedule. The upper bound and \Cref{lem:app-competitive} give its mesh bound. By \eqref{eq:app-regret-expansions} and \eqref{eq:app-regret-result},
\begin{equation}
\label{eq:app-optimal-scaled-cost}
\frac1{TL}\sum_{k=0}^{K_T^\star}\mu(u_k)h_k^2\longrightarrow B.
\end{equation}
Fix $\delta\in(0,1)$ and call a risk-bearing interval short if $I(u_k)h_k\leq(1-\delta)L$. Such an interval has length at most $CL$. By \eqref{eq:app-local-expansions}, its exact Gaussian exponent is at most $(1-\delta)L+O(L^2/T)$. The normal-tail lower bound therefore gives risk at least $cL^{-1/2}e^{-(1-\delta)L}$. Feasibility bounds the number $N_{\mathrm{short}}$ of these intervals by
\begin{equation}
\label{eq:app-short-count}
N_{\mathrm{short}}\leq C\epsilon_T\sqrt L\,e^{(1-\delta)L}
=CT\sqrt L\,e^{-\delta L}=o(T/L).
\end{equation}
Their total calendar fraction is $o(1)$ because each is at most $CL$. The terminal fraction is also $o(1)$ by \eqref{eq:app-mesh}.

Let $G_\delta$ contain the remaining risk-bearing intervals and put $z_k=I(u_k)h_k/L$. Then $z_k>1-\delta$ on $G_\delta$, and
$$
\sum_{k\in G_\delta}\frac{h_k}{T}A(u_k)\longrightarrow B,
\qquad
\sum_{k\in G_\delta}\frac{h_k}{T}I(u_k)\longrightarrow D.
$$
Both statements follow from the mesh bound, smoothness, and the negligible calendar fractions just removed. Using nonnegativity of the omitted costs in \eqref{eq:app-optimal-scaled-cost}, we get
\begin{equation}
\label{eq:app-excess-length}
\sum_{k\in G_\delta}\frac{h_k}{T}A(u_k)\big\{z_k-(1-\delta)\big\}\leq\delta B+o(1).
\end{equation}
For $\eta>0$, let $W_{\mathrm{long}}$ be the calendar fraction of intervals in $G_\delta$ with $z_k>1+\eta$. Since $A_-:=\min A>0$, \eqref{eq:app-excess-length} implies
$$
\limsup_{T\to\infty}W_{\mathrm{long}}\leq\frac{\delta B}{A_-(\eta+\delta)}.
$$
We can now bound the count from both sides. From $z_k>1-\delta$ and \eqref{eq:app-short-count},
$$
\limsup_{T\to\infty}\frac{LK_T^\star}{T}
=\limsup_{T\to\infty}\sum_{k\in G_\delta}\frac{h_k}{T}\frac{I(u_k)}{z_k}
\leq\frac D{1-\delta}.
$$
For the lower bound, retain only $k\in G_\delta$ with $z_k\leq1+\eta$ and use $I_+:=\max I$:
$$
\liminf_{T\to\infty}\frac{LK_T^\star}{T}
\geq\frac1{1+\eta}\left\{D-\frac{I_+\delta B}{A_-(\eta+\delta)}\right\}.
$$
Letting $\delta\downarrow0$ and then $\eta\downarrow0$ makes the two bounds equal to $D$. Thus
\begin{equation}
\label{eq:app-count-result}
K_T^\star\sim\frac{T}{L_T}\int_0^1I(u)\,\d u,
\end{equation}
which proves the count assertion in \Cref{cor:regret} for an optimal integer schedule, not just for the upper-bound construction.
\end{proof}

\section{Experimental Details}
\label{app:exp}

We describe the data generation, policy learning, value evaluation, and scheduling procedures implemented in the experiment scripts. The synthetic and clinical learning experiments supply estimated candidate-value trajectories to the same scheduling routines. The homogeneous Gaussian experiment instead specifies the candidate-value process directly, which also permits evaluation of a selected schedule using analytic population rewards and risks.

\subsection{Synthetic learning environments}
\label{app:exp-synthetic}

\paragraph{Observations and treatment effects.} For each of the low-, medium-, and high-SNR settings in \Cref{fig:frontier,fig:bifurcation}, we generate $N=400$ learning trajectories, each containing $300$ observations. We use $r$ for the observation index and $t$ for the policy checkpoint: a candidate is retained after every five observations, giving $T=60$ candidates at sample sizes $5,10,\ldots,300$. For each trajectory and observation, the five covariates are independent and uniformly distributed on $[-1,1]$, treatment is assigned independently with probability $1/2$, and the standardized outcome innovation is Gaussian:
\begin{equation}
\label{eq:exp-synthetic-draws}
X_{rj}\overset{\mathrm{i.i.d.}}{\sim}\operatorname{Unif}(-1,1),\quad j=1,\ldots,5,\qquad A_r\sim\operatorname{Bernoulli}(1/2),\qquad \xi_r\sim\mathcal N(0,1).
\end{equation}
Writing $s$ for the SNR setting, we define
\begin{equation}
\label{eq:exp-synthetic-outcomes}
\begin{aligned}
m(x)&=0.50x_3+0.25x_4-0.25x_5, & \tau^\star(x)&=\sqrt{3/2}\,(x_1+x_2),\\
Y_{s,r}(a)&=m(X_r)+(a-1/2)\tau^\star(X_r)+\sigma_{Y,s}\xi_r, & Y_{s,r}&=Y_{s,r}(A_r).
\end{aligned}
\end{equation}
Thus, all three settings share the same conditional treatment effect and the same optimal treatment rule $\pi^\star(x)=\mathbf 1\{\tau^\star(x)\geq0\}$. The generator uses seed $20260922$ and common draws of $X_r$, $A_r$, and $\xi_r$ across settings. Trajectories are generated independently within each setting, while corresponding trajectories across settings are paired through these common random numbers. Both potential outcomes use the same innovation within an observation; only the observed outcome enters policy training.

\paragraph{Sequential policy learner.} Let $\phi(x)=(1,x^\top)^\top\in\mathbb R^6$. We fit the outcome model $\widehat m_r(x,a)=\beta_r^\top\phi(x)+(a-1/2)\gamma_r^\top\phi(x)$ by a single stochastic-gradient update after each observation, starting from $\beta_0=\gamma_0=0$. With $\lambda_{\mathrm{syn}}=10^{-4}$, the implemented updates are
\begin{equation}
\label{eq:exp-synthetic-learner}
\begin{aligned}
e_r&=Y_{s,r}-\beta_{r-1}^\top\phi(X_r)-(A_r-1/2)\gamma_{r-1}^\top\phi(X_r),\\
\eta_{s,r}&=\eta_{0,s}(1+r/100)^{-0.6},\\
\beta_r&=(1-\eta_{s,r}\lambda_{\mathrm{syn}})\beta_{r-1}+\eta_{s,r}e_r\phi(X_r),\\
\gamma_r&=(1-\eta_{s,r}\lambda_{\mathrm{syn}})\gamma_{r-1}+\eta_{s,r}e_r(A_r-1/2)\phi(X_r).
\end{aligned}
\end{equation}
Both updates use the residual computed before updating either coefficient vector, and the penalty applies to all their coordinates, including the intercepts. The retained policy at checkpoint $t$ is $\pi_{s,t}(x)=\mathbf 1\{\gamma_{5t}^\top\phi(x)\geq0\}$. The first retained candidate processed five observations. These checkpoints group five successive one-observation updates. The settings are configured as
\begin{center}
\begin{tabular}{lrr}
\hline
Setting & Initial learning rate $\eta_{0,s}$ & Outcome-noise scale $\sigma_{Y,s}$\\
\hline
Low SNR & $0.030$ & $7.0$\\
Medium SNR & $0.150$ & $2.5$\\
High SNR & $0.200$ & $0.25$\\
\hline
\end{tabular}
\end{center}

\paragraph{Candidate-value evaluation.} We draw a common bank of $M=8192$ evaluation covariates independently of the training observations, using the same covariate distribution. The bank is reused for every checkpoint, trajectory, and SNR setting. Since the treatment effect is known in this experiment, we evaluate the policy-dependent outcome contrast directly and normalize it by the corresponding contrast of the oracle policy:
\begin{equation}
\label{eq:exp-synthetic-value}
\widehat V_{s,t}^{(n)}=
\frac{M^{-1}\sum_{\ell=1}^{M}\bigl[\pi_{s,t}^{(n)}(X_\ell^{\mathrm{ev}})-1/2\bigr]\tau^\star(X_\ell^{\mathrm{ev}})}{M^{-1}\sum_{\ell=1}^{M}\bigl[\pi^\star(X_\ell^{\mathrm{ev}})-1/2\bigr]\tau^\star(X_\ell^{\mathrm{ev}})}.
\end{equation}
The omitted prognostic term is common to every policy and therefore cancels in value differences. No evaluation outcomes or additional outcome noise are generated for this calculation. The resulting values are finite-bank estimates of a normalized population contrast, not exact population integrals. This policy-value normalization precedes, and differs from, the never-update versus always-update normalization of the scheduling objective described below.

The generator also summarizes the first $25$ checkpoint increments. It averages their across-trajectory means and divides by the square root of the average across-trajectory sample variances. The saved empirical SNRs are approximately $0.0496$, $0.1145$, and $0.2816$ for the low, medium, and high settings, respectively.

\paragraph{Homogeneous Gaussian experiment.} The phase-diagram script associated with \Cref{fig:phase-diagram} is configured with $T=512$, safety budget $\epsilon=0.35$, and risk-grid width $\Delta=0.001$. It uses $50$ equally spaced drift values in $[0.5,3.0]$ and $50$ equally spaced volatility values in $[0.25,1.5]$. At every parameter pair, the experiment uses three data replications, each containing $4000$ independent trajectories. With candidates indexed here as $1,\ldots,T$, the simulated process is
\begin{equation}
\label{eq:exp-homogeneous-process}
V_1^{(b,n)}=0.15,\qquad V_{t+1}^{(b,n)}=V_t^{(b,n)}+\frac{\mu}{T}+\frac{\sigma}{T}Z_t^{(b,n)},\qquad t=1,\ldots,T-1,
\end{equation}
where the $Z_t^{(b,n)}$ are independent standard normal variables. The innovation generator uses seed $20260921$ and reuses the same innovation bank across all parameter pairs. It stores $T$ innovation rows per trajectory; only the first $T-1$ rows are needed because the terminal node has no policy value. In particular, the per-step noise standard deviation in \eqref{eq:exp-homogeneous-process} is $\sigma/T$, not $\sigma/\sqrt T$.

For each replication, the script estimates edge rewards and strict unsafe-update frequencies from its $4000$ value trajectories and solves the conservatively gridded dynamic program. It then evaluates that selected path using the analytic population matrices
\begin{equation}
\label{eq:exp-homogeneous-oracle}
\boldsymbol R_{ij}=\frac{j-i}{T}\left(0.15+\frac{\mu(i-1)}{T}\right),\qquad
\boldsymbol S_{ij}=\Phi\!\left(-\frac{\mu\sqrt{j-i}}{\sigma}\right)\quad(j\leq T),\qquad
\boldsymbol S_{i,T+1}=0.
\end{equation}
The displayed heatmaps average the three replication-specific path summaries. For a path with $K>0$ actual updates, the summaries are
\begin{equation}
\label{eq:exp-heatmap-summaries}
\bar p=\frac1K\sum_{k=0}^{K-1}\boldsymbol S_{t_k,t_{k+1}},\qquad
\bar h=\frac1{KT}\sum_{k=0}^{K-1}(t_{k+1}-t_k),\qquad
J=\sum_{k=0}^{K}\boldsymbol R_{t_k,t_{k+1}}.
\end{equation}
The initial and terminal nodes are excluded from $K$, and the terminal hold is excluded from both pre-update averages. If the selected path makes no updates, the display convention is $K=0$, $\bar p=0$, and $\bar h=1$. Thus, the plotted interval is a fraction of the horizon, whereas $h_T$ in the theory is measured in periods. The script also records total population risk, unrounded empirical risk, the coefficient of variation of pre-update intervals, and the indicator that population risk exceeds $\epsilon$. An oracle gridded schedule is computed separately from the population matrices. The exact unrounded oracle-label solver is optional and is not enabled in the default run.

\paragraph{Theoretical contours and correlations.} The default contours use the leading homogeneous formulas, with $I=\mu^2/(2\sigma^2)$ and $L=\log(T/\epsilon)$:
\begin{equation}
\label{eq:exp-heatmap-reference}
\begin{aligned}
h_{\mathrm{ref}}&=L/I, & K_{\mathrm{ref}}&=TI/L, & p_{\mathrm{ref}}&=\epsilon L/(TI),\\
\bar h_{\mathrm{ref}}&=L/(TI), & J_{\mathrm{ref}}&=J_{\mathrm{AU}}-\frac{\mu L}{2TI}, & J_{\mathrm{AU}}&=0.15+\frac{\mu(T-1)}{2T}.
\end{aligned}
\end{equation}
The default does not include the $\log L$ correction. A reference cell is masked if $h_{\mathrm{ref}}<5$, $h_{\mathrm{ref}}/T>0.15$, or its reference regret is negative; it is not clipped to an integer spacing or update count. The configured contour levels are $10,20,40,60$ for $K$, $0.005,0.01,0.02,0.04$ for $\bar p$, $0.01,0.02,0.04,0.08$ for $\bar h$, and $0.35,0.45,0.55,0.65$ for $J$, retaining only levels inside the finite range of the corresponding reference surface. For each quantity and solver, Pearson correlation is computed after flattening the cells on which both the reference and the solver summary are finite. The three replications are averaged before this comparison, rather than being treated as separate grid cells.

\paragraph{Introductory illustration.} \Cref{fig:intro} is constructed from the medium-SNR value bank. It selects the trajectory minimizing the mean squared distance to the ensemble mean after dividing each checkpoint's deviation by its across-trajectory sample standard deviation, with a denominator floor of $10^{-12}$. In the saved bank this is trajectory $361$, using zero-based trajectory identifiers. The three schedules update every $24$, $12$, and $6$ checkpoints. Unsafe markers compare successive deployed values strictly on the selected realized trajectory, whereas the displayed fractions of expected gain captured use the across-trajectory mean value sequence and the never-update/always-update normalization. Those fractions are clipped to $[0,1]$ for this illustration only.

\subsection{International Stroke Trial analysis}
\label{app:exp-ist}

\paragraph{Population, treatment, and outcome.} We use the corrected version-2 IST file \texttt{IST\_corrected.csv}. Its $19{,}435$ records yield $19{,}285$ analysis records after removing $150$ observations with missing or unknown six-month outcome codes: $97$ with \texttt{OCCODE}$=0$ and $53$ with \texttt{OCCODE}$=9$. Treatment is the randomized aspirin assignment, $A=1$ for \texttt{RXASP}=\texttt{Y} and $A=0$ for \texttt{RXASP}=\texttt{N}, with known propensity $1/2$. We define $Y=1$ for \texttt{OCCODE}$\in\{3,4\}$, corresponding to being alive and independent at six months, and $Y=0$ for codes $1$ or $2$, corresponding to death or dependence. The retained sample contains $9639$ aspirin-assigned and $9646$ non-aspirin-assigned patients, with $7160$ positive outcomes. We do not restrict this sample by a post-randomization final diagnosis. Concurrent randomized heparin assignment, treatment adherence, post-randomization outcomes, identifiers, geographic variables, and dates are not included among the policy covariates.

\paragraph{Covariates and preprocessing.} The cleaning script produces $35$ numeric covariates with the following encodings. Binary \texttt{Y}/\texttt{N} fields are coded as $1/0$; missing values in the binary clinical fields are filled with zero, with separate missingness indicators for atrial fibrillation, prior heparin, and prior aspirin.
\begin{center}
\begin{tabular}{p{0.32\linewidth}p{0.54\linewidth}r}
\hline
Source fields & Encoding & Count\\
\hline
\texttt{AGE}, \texttt{RDELAY}, \texttt{RSBP} & Age, delay to randomization, and systolic blood pressure in their original units & 3\\
\texttt{SEX} & Indicator for female sex & 1\\
{\raggedright\texttt{RSLEEP}, \texttt{RATRIAL}, \texttt{RCT}, \texttt{RVISINF}, \texttt{RHEP24}, \texttt{RASP3}\par} & Six binary clinical indicators and three missingness indicators & 9\\
\texttt{RCONSC} & Drowsy and unconscious indicators; fully alert is the reference & 2\\
\texttt{RDEF1}--\texttt{RDEF8} & Present and cannot-assess indicators for each deficit; absent is the reference & 16\\
\texttt{STYPE} & Indicators for \texttt{PACS}, \texttt{TACS}, \texttt{POCS}, and \texttt{OTH}; \texttt{LACS} is the reference & 4\\
\hline
\end{tabular}
\end{center}
The cleaner checks that the resulting model fields are numeric and complete, and rejects missing continuous measurements rather than imputing them. It creates a reproducibly shuffled stream using seed $20260922$. This ordering is a simulation device, not a clinical arrival-time ordering.

We split the retained records into $15{,}428$ development records and $3857$ evaluation records using a $20\%$ evaluation fraction, split seed $20260921$, and joint treatment--outcome strata encoded by $2A+Y$. Every covariate is centered and scaled using the development mean and population-standard-deviation convention, and the same transformation is applied to evaluation covariates. Coordinates with development standard deviation at most $10^{-10}$ would be removed; none is removed for the supplied split. The transformation is fixed before constructing learning trajectories, rather than recomputed at each checkpoint.

\paragraph{Bootstrap learning trajectories and checkpoints.} Using random-number seed $20260922$, we draw $N=500$ bootstrap streams with replacement from the development pool. Each stream contains $15{,}250$ sampled records. Along a stream, we fit policies on nested prefixes of sizes
\begin{equation}
\label{eq:exp-ist-checkpoints}
n_t=4000+250(t-1),\qquad t=1,\ldots,46.
\end{equation}
Consequently, the clinical scheduling horizon is $T=46$ policy checkpoints, not $60$ patient observations. All $15{,}428$ development records remain eligible for sampling; $15{,}250$ is the largest fitted bootstrap sample size. Bootstrap streams can share patient records, and a record can appear more than once within a stream. Every fitted policy is evaluated on the same held-out evaluation sample. The separate disjoint length-$60$ trajectory file written by the cleaner is not used by the frontier or schedule-diagram scripts; those scripts load \texttt{ist\_safe\_policy\_stream.csv} and construct the bootstrap streams just described.

\paragraph{Ridge outcome model and treatment rule.} For standardized covariates $x\in\mathbb R^{35}$, define $z(x,a)=(1,x^\top,a,ax^\top)^\top\in\mathbb R^{72}$. At checkpoint $t$ of trajectory $n$, we use the ridge estimate
\begin{equation}
\label{eq:exp-ist-ridge}
\widehat b_t^{(n)}=
\left(\sum_{r=1}^{n_t}z_rz_r^\top+5P\right)^{-1}
\sum_{r=1}^{n_t}z_rY_r,
\qquad P=\operatorname{diag}(0,1,\ldots,1).
\end{equation}
The cross-products are updated cumulatively as each block of $250$ observations arrives. The penalty is applied to the unnormalized sum-of-squares problem and does not penalize the overall intercept. For evaluation, the arm-specific predictions are clipped to $[0,1]$ before defining the policy:
\begin{equation}
\label{eq:exp-ist-policy}
\widehat m_{t,a}^{(n)}(x)=\min\!\left\{1,\max\!\left\{0,z(x,a)^\top\widehat b_t^{(n)}\right\}\right\},\qquad
\pi_t^{(n)}(x)=\mathbf 1\!\left\{\widehat m_{t,1}^{(n)}(x)>\widehat m_{t,0}^{(n)}(x)\right\}.
\end{equation}
Thus ties between the clipped predictions assign no aspirin. The fitted outcome model supplies both the treatment rule and the nuisance predictions for evaluation; a separate outcome learner is not fitted for the evaluation score.

\paragraph{Held-out value estimates.} For each fitted candidate, we compute its augmented inverse-propensity-weighted value on the common evaluation sample $\mathcal E$:
\begin{equation}
\label{eq:exp-ist-aipw}
\widehat V_t^{(n)}=
\frac1{|\mathcal E|}\sum_{\ell\in\mathcal E}
\left[
\widehat m_{t,\pi_t^{(n)}(X_\ell)}^{(n)}(X_\ell)
+\frac{\mathbf 1\{A_\ell=\pi_t^{(n)}(X_\ell)\}}{p_{\pi_t^{(n)}(X_\ell)}}
\left\{Y_\ell-\widehat m_{t,A_\ell}^{(n)}(X_\ell)\right\}
\right],\qquad p_0=p_1=1/2.
\end{equation}
The final score and its average are not clipped. Neither policy fitting nor covariate standardization uses evaluation outcomes. However, the resulting value estimates are subsequently used both to select schedules and to report their frontier values; this experiment does not use another independent evaluation sample for the selected meta-policy.

\subsection{Baselines, budgets, and reporting}
\label{app:exp-reporting}

\paragraph{Shared scheduling inputs.} Within each setting, let $\widehat V_i^{(n)}$ denote the saved or estimated value of candidate $i$ in trajectory $n$. The frontier and schedule-diagram scripts use the direct mean $\widehat\nu_i=N^{-1}\sum_n\widehat V_i^{(n)}$ and the empirical matrices
\begin{equation}
\label{eq:exp-edge-estimates}
\widehat{\boldsymbol R}_{ij}=\frac{j-i}{T}\widehat\nu_i,\qquad
\widehat{\boldsymbol S}_{ij}=\frac1N\sum_{n=1}^{N}\mathbf 1\!\left\{\widehat V_j^{(n)}<\widehat V_i^{(n)}\right\}\quad(j\leq T),\qquad
\widehat{\boldsymbol S}_{i,T+1}=0.
\end{equation}
The implementation uses zero-based learned-policy nodes $0,\ldots,T-1$ and a separate sink $T$, equivalent to the paper's nodes $1,\ldots,T$ and sink $T+1$. In particular, the final learned candidate can contribute reward for its last deployment period. All methods use the same mean values and edge risks. A tie in estimated candidate values is not counted towards an unsafe update.

The common risk-grid width is $\Delta=0.001$. Write $C_{ij}=\Delta\lceil\widehat{\boldsymbol S}_{ij}/\Delta\rceil$ for the rounded resource in risk units, with terminal resource zero. The proposed method applies \Cref{alg:main} using these upward-rounded costs and a downward-rounded budget. Its implementation accumulates $(j-i)\widehat\nu_i$, rather than $(j-i)\widehat\nu_i/T$; this common factor of $T$ does not change the selected schedule or the normalized objective. It computes the frontier up to the rounded cost $\epsilon_{\mathrm{AU}}=\sum_{i=1}^{T-1}C_{i,i+1}$ of updating at every checkpoint, then reuses the retained terminal labels at the plotted budgets.

\paragraph{Baseline definitions.} The \emph{periodic} baseline enumerates every integer period $h\in\{1,\ldots,T\}$, deploys candidates $1,1+h,1+2h,\ldots$ while their indices do not exceed $T$, and then holds the final incumbent through the terminal node. It retains nondominated candidates with total rounded resource at most $\epsilon_{\mathrm{AU}}$, including the no-update schedule, and selects the highest-value feasible candidate at each plotted budget.

The \emph{equal-risk} baseline allows nonperiodic timing, but requires a common risk allowance for all actual updates. For a path $P$ with $K>0$ updates, its required budget is
\begin{equation}
\label{eq:exp-equal-risk}
\epsilon_{\mathrm{eq}}(P)=K\max_{(i,j)\in P,\,j\leq T}C_{ij}.
\end{equation}
Thus feasibility at $\epsilon$ is equivalent to requiring every actual transition to cost at most $\epsilon/K$. A dynamic program enumerates update counts $K=1,\ldots,T-1$ and tracks the largest edge resource and cumulative value, with dominance pruning within each count layer. The no-update path is included separately.

The \emph{earliest-safe} baseline sweeps a threshold over zero and every distinct finite rounded resource of an actual transition. For each threshold, it starts at the initial candidate and repeatedly selects the earliest later candidate whose edge resource does not exceed the threshold; if none exists, it retains the incumbent to the terminal node. Duplicate paths are removed. The method then selects the highest-value generated path whose total additive rounded resource fits the plotted budget. The threshold is therefore tuned over the complete finite set at which the greedy path can change, not over an externally chosen coarse threshold grid.

In the frontier routines, a Bellman predecessor is replaced only when its value improves by more than $10^{-12}$, so tied candidates retain the first predecessor encountered in increasing node order. The separate heatmap dynamic program uses $10^{-13}$ for this comparison. Frontier pruning sorts by increasing resource and decreasing value, and selection among tied feasible frontier values favors smaller resource.

\paragraph{Safety-budget grids.} For each of the four frontier panels, let $\epsilon_{\mathrm{first}}$ be the smallest rounded resource among the proposed method's retained nondominated labels that contain an actual update. This is a landmark of the retained frontier. Across panels, the common axis endpoints are
\begin{equation}
\label{eq:exp-budget-axis}
\epsilon_{\min}=\max\!\left\{0.92\min_s\epsilon_{\mathrm{first},s},\Delta/10\right\},\qquad
\epsilon_{\max}=\max\!\left\{1.08\max_s\epsilon_{\mathrm{AU},s},1.01\epsilon_{\min}\right\},
\end{equation}
where infinite first-update landmarks are omitted. For the saved outputs, the first-update landmarks for low SNR, medium SNR, high SNR, and IST are $0.203$, $0.030$, $0$, and $0.432$, respectively; their always-update costs are $27.536$, $27.559$, $24.928$, and $22.492$. These give the common range $[10^{-4},29.76372]$. The frontier plot uses $241$ budget points, and the schedule diagram uses $90$ points over the same range. Both are equally spaced after a base-$10$ symmetric-log transformation with linear threshold $\epsilon_{\min}$; since all displayed budgets are at or above that threshold, they are logarithmically spaced on the positive range. The schedule-diagram subsampling factor is one, so all $90$ points are retained.

\paragraph{Objective normalization and update coordinates.} For a schedule $P$, let $\widehat J_T(P)=\sum_{(i,j)\in P}\widehat{\boldsymbol R}_{ij}$. The plotted objective is
\begin{equation}
\label{eq:exp-frontier-normalization}
\widetilde J_T(P)=\frac{\widehat J_T(P)-\widehat\nu_1}{T^{-1}\sum_{i=1}^{T}\widehat\nu_i-\widehat\nu_1},
\end{equation}
which maps never updating to zero and updating at every checkpoint to one. The frontier script does not clip this quantity to one: selectively retaining candidates can outperform the always-update comparator under the unprojected empirical means. The schedule diagram plots each selected interior node at its zero-based checkpoint index divided by $T$; in the paper's one-based notation, candidate $i$ is displayed at $(i-1)/T$. The first candidate is treated as available at the start of the scheduling horizon, after its initial training observations have been collected. The two vertical landmarks in each schedule panel are $\epsilon_{\mathrm{first}}$ and $\epsilon_{\mathrm{AU}}$.

\paragraph{Replication and uncertainty.} The synthetic frontier panels each use one bank of $400$ learning trajectories, and the IST panel uses one bank of $500$ bootstrap learning trajectories with the fixed evaluation split above. Schedule selection and frontier reporting use the same estimated matrices; there is no additional independent trajectory bank for evaluating the selected schedules in these plots. The homogeneous experiment differs in that it evaluates selected paths using the population matrices in \eqref{eq:exp-homogeneous-oracle}. The frontier script defaults to zero outer bootstrap replications and therefore reports curves without uncertainty bands. Its optional outer bootstrap resamples whole columns of a value-trajectory matrix with replacement, re-estimates the matrices, and re-solves every method. When enabled, the displayed band is the original normalized curve plus or minus one sample standard deviation across those bootstrap curves, not a confidence interval; the outer-bootstrap random-number seed is $20260923$.

\end{document}